\documentclass[journal]{IEEEtran}
\usepackage{array}
\usepackage[caption=false,font=normalsize,labelfont=sf,textfont=sf]{subfig}
\usepackage{textcomp}
\usepackage{stfloats}
\usepackage{url}
\usepackage{verbatim}
\usepackage{graphicx}
\graphicspath{{figures/}}
\usepackage{float,amsmath,amsthm,amsfonts,xcolor,graphicx}
\usepackage{booktabs}
\usepackage{multirow}
\usepackage{dsfont}
\usepackage{cite}
\usepackage[ruled,vlined,linesnumbered]{algorithm2e}
\usepackage{algpseudocode}
\usepackage{lipsum}
\usepackage[english]{babel}
\usepackage{booktabs}
\usepackage{multirow}
\usepackage{dsfont}
\usepackage{cite}
\usepackage{amssymb}
\newtheorem{problem}{Problem}[section]
\newtheorem{theorem}{Theorem}[section]
\newtheorem{example}{Example}[section]
\newtheorem{corollary}{Corollary}[section]

\newtheorem{lemma}{Lemma}[section]
\newtheorem{definition}{Definition}[section]

\DeclareMathOperator*{\argmax}{arg\,max}

\def\BibTeX{{\rm B\kern-.05em{\sc i\kern-.025em b}\kern-.08em
    T\kern-.1667em\lower.7ex\hbox{E}\kern-.125emX}}
\usepackage{balance}

\title{RRT$^\eta$: Sampling-based Motion Planning and Control from STL Specifications using Arithmetic-Geometric Mean Robustness}
\author{Ahmad Ahmad
\IEEEmembership{Member, IEEE}, Shuo Liu
\IEEEmembership{Member, IEEE}, Roberto Tron
\IEEEmembership{Member, IEEE}, Calin Belta
\IEEEmembership{Fellow, IEEE}
\thanks{A. Ahmad, and R. Tron are with the Division of Systems Engineering, Boston University, Boston, MA, USA.}
\thanks{R. Tron and S. Liu are with the Department of Mechanical Engineering, Boston University, Boston, MA, USA.}
\thanks{C. Belta is with the Department of Electrical and Computer Engineering and the Department of Computer Science at the University of Maryland, College Park, MD, USA.}
}

\begin{document}
\maketitle

\begin{abstract}
Sampling-based motion planning has emerged as a powerful approach for robotics, enabling exploration of complex, high-dimensional configuration spaces. When combined with Signal Temporal Logic (STL)—a temporal logic widely used for formalizing interpretable robotic tasks—these methods can address complex spatiotemporal constraints. However, traditional approaches rely on min-max robustness measures that focus only on critical time points and subformulae, creating non-smooth optimization landscapes with sharp decision boundaries that hinder efficient tree exploration.

We propose RRT$^\eta$, a sampling-based planning framework that integrates the Arithmetic-Geometric Mean (AGM) robustness measure to evaluate satisfaction across all time points and subformulae. Our key contributions include: (1) AGM robustness interval semantics for reasoning about partial trajectories during tree construction, (2) an efficient incremental monitoring algorithm computing these intervals, and (3) enhanced Direction of Increasing Satisfaction vectors leveraging Fulfillment Priority Logic (FPL) for principled objective composition. Our framework synthesizes dynamically feasible control sequences satisfying STL specifications with high robustness while maintaining the probabilistic completeness and asymptotic optimality of RRT$^\ast$. We validate our approach on three robotic systems. A double integrator point robot, a unicycle mobile robot, and a 7-DOF robot arm, demonstrating superior performance over traditional STL robustness-based planners in multi-constraint scenarios with limited guidance signals.
\end{abstract}

\begin{IEEEkeywords}
Motion Planning, Formal Methods, Temporal Logic, Manipulators, Mobile Robots
\end{IEEEkeywords}

\section{Introduction}


\IEEEPARstart{S}ampling-based methods have significantly advanced motion planning and control synthesis for robotic systems in complex, high-dimensional environments. Algorithms such as Rapidly-exploring Random Trees (RRT) \cite{lavalle1998rapidly} and variants \cite{otte2016rrtx,Kobilarov2012CERRTstar,Wu2020e,KaramanRRTstarIJRR,ahmadcbfrrt*,yang2023lqrcbfrrtstar} efficiently explore configuration spaces without explicit discretization, particularly useful for nonlinear dynamics. RRT$^\ast$ \cite{KaramanRRTstarIJRR} extends RRT with rewiring to ensure asymptotic convergence to optimal paths. Recently, Temporal Logics (TLs) \cite{baier2008principles} have become powerful tools for specifying complex robotic tasks \cite{hadas2018synthesis_review,sadana2023survey,Cristi2017TWTL,belta2017formal}. Signal Temporal Logic (STL) \cite{maler2004STLpaper} offers an expressive formalism for interpretable temporal properties, enabling rich constraints with explicit timing such as "Visit region A between times $t_1$ and $t_2$, then visit region B within $t_3$ time units while always avoiding region C."

Synthesizing motion plans under STL specifications presents computational challenges favoring sampling-based approaches. Grid-based methods suffer from dimensionality curse while trajectory optimization struggles with non-convex landscapes from discrete temporal operators. Sampling-based methods circumvent these issues by avoiding explicit state space discretization. However, effective integration requires maximizing robustness to specification satisfaction throughout trajectories. Traditional min-max robustness metrics focus solely on critical time points and subformulae, creating non-smooth optimization landscapes with sharp decision boundaries. This brittleness, observed across domains including reinforcement learning \cite{ahmad2024APPO}, causes small trajectory changes to produce abrupt shifts in critical constraints, providing inconsistent planner guidance. Moreover, principled composition of competing temporal objectives remains open. These issues intensify in applications requiring high-confidence autonomy with time-sensitive goals and strict safety requirements.

Researchers have increasingly integrated sampling-based algorithms with formal specifications to address complex planning problems. These specifications range from basic safety properties to rich spatiotemporal constraints. For safety specifications, approaches combining Control Barrier Functions with sampling-based planners \cite{ahmadcbfrrt*,yang2023lqrcbfrrtstar,cbfrrt_app_Fainekos_ppr} synthesize controllers guaranteeing set invariance while eliminating explicit collision checking. Temporal logic specifications, however, enable more general spatiotemporal task descriptions, encoding complex sequences of goals with explicit timing constraints. For such specifications, two main paradigms have emerged. Automata-based approaches \cite{Cristi2020_TWTLrrt,Cristi2013RRG,Cristi2020ijrrRRGLTL} leverage finite-state abstractions to guide tree construction toward specification satisfaction, but provide only binary guarantees without quantitative robustness measures. Robustness-based methods instead use quantitative semantics to both verify satisfaction strength and guide synthesis. Within this paradigm, optimization-based approaches \cite{Sadra_RobustSTL_MPC,VRaman_MPC_STL} formulate the problem as Mixed-Integer Linear Programs (MILP) or nonlinear optimization, while sampling-based methods like STL-RRT$^\ast$ \cite{CristiKaraman17_STL_RRTstar} guide tree exploration through the Direction of Increasing Satisfaction (DIS) vector. However, both optimization and sampling-based robustness methods typically rely on traditional min-max metrics that evaluate satisfaction based solely on the most critical time points and subformulae. This creates two key limitations: first, the resulting non-smooth optimization landscape provides inconsistent gradient information, and second, when composing objectives from multiple subformulae, simple selection mechanisms fail to balance competing requirements adequately. The challenge of principled objective composition extends beyond temporal logic, where recent work in Fulfillment Priority Logic (FPL) \cite{mabsout2025FPL} has demonstrated nuanced decision-making frameworks that prioritize less-fulfilled objectives.

We address these challenges through RRT$^\eta$, integrating Arithmetic-Geometric Mean (AGM) \cite{cristi2019AGMstl} robustness into RRT$^\ast$. Unlike min-max robustness focusing on critical points, AGM evaluates satisfaction holistically across all time points and subformulae, creating smoother optimization landscapes with consistent gradient information while maintaining probabilistic completeness and asymptotic optimality. 

Our contributions are threefold. First, we develop AGM robustness interval semantics for partial trajectories with efficient incremental monitoring (Section \ref{sec:AGMMonitor}), enabling informed exploration before complete paths are formed. Second, we introduce enhanced Direction of Increasing Satisfaction vectors that leverage AGM's smooth gradients and integrate Fulfillment Priority Logic \cite{mabsout2025FPL} for principled multi-objective composition (Section \ref{sec:RRTeta}). Third, we prove that RRT$^\eta$ maintains the probabilistic completeness and asymptotic optimality guarantees of RRT$^\ast$ (Section \ref{subsec:ARMORRRTeta_completenss}). We validate our approach on double integrator, unicycle, and 7-DOF robot arm systems (Section \ref{sec:example}), demonstrating higher robustness and smoother exploration behavior than min-max approaches while synthesizing dynamically feasible control sequences for high-confidence autonomy applications.

\section{Preliminaries}\label{sec:prelim}

\textbf{Signal Temporal Logic}. STL was introduced in \cite{maler2004STLpaper} to specify and monitor properties of signals over time. In this work, we consider discrete-time signals where signal valuations $s_t\in\mathbb{R}$ are sampled at discrete time points $t\in \mathbb{Z}_{\geq 0}$. We denote a finite signal segment from time $t_1$ to $t_2$ as $\mathbf{s}_{t_1,t_2}:=s_{t_1}s_{t_1+1}\dots s_{t_2}$. An STL formula $\phi$ is defined recursively as:

\begin{align}
\phi ::= \mu \mid \neg\phi \mid \phi_1 \wedge \phi_2 \mid \phi_1 \lor \phi_2 \mid \mathbf{G}_{[a,b]}\phi_1 \mid \mathbf{F}_{[a,b]}\phi_1 
\end{align}
where $\mu$ is a predicate in the form of $h(s_t) > \varsigma$, $\varsigma \in\mathbb{R}$, $h:\;\mathbb{R}\to\mathbb{R}$ is a linear function; $\mathbf{G}$ and $\mathbf{F}$ represent the \textit{Globally} (Always) and \textit{Finally} (Eventually) operators, respectively, where $a,b\in\mathbb{Z}_{\geq 0}$ and $a<b$. We denote the set of all possible STL formulae as $\Phi$. 


STL admits both Boolean and quantitative semantics. Under \textit{Boolean semantics}, a signal either satisfies or violates a formula. The temporal operators have intuitive interpretations: $\mathbf{G}_{[a,b]}\phi$ means that $\phi$ must hold at all time points in the interval $[a,b]$, while $\mathbf{F}_{[a,b]}\phi$ requires $\phi$ to hold at some time point in $[a,b]$. Conjunction $\phi_1 \wedge \phi_2$ requires both subformulae to hold, while disjunction $\phi_1 \vee \phi_2$ requires at least one to hold. The formal Boolean semantics are detailed in \cite{maler2004STLpaper}.

 Beyond Boolean satisfaction, STL admits quantitative robustness semantics that measure how strongly a signal satisfies a specification. The \textit{traditional robustness} \cite{donze2010robustSTL} uses min-max operations: for conjunction, it takes the minimum robustness of subformulae; for disjunction, the maximum; for $\mathbf{G}_{[a,b]}$, the minimum over all time points; and for $\mathbf{F}_{[a,b]}$, the maximum. This focuses on the most critical time points and subformulae. In contrast, AGM robustness \cite{cristi2019AGMstl} evaluates satisfaction holistically across all time points and subformulae through arithmetic and geometric means, producing smoother landscapes for optimization-based planning.


\textbf{STL AGM robustness semantics}. Given $F:\mathbb{R}\to\mathbb{R}$, let $[F]_{+}:=
\begin{cases}
F, & F>0,\\
0, & \text{otherwise},
\end{cases}~
[F]_{-}:=-[-F]_{+},$ where $F=[F]_{+}+[F]_{-}$. We define AGM operators for disjunction and conjunction that aggregate robustness values. Given $N$ robustness values $r_i\in\mathbb{R}$, $i=1,\ldots,N$, where $N$ represents either the number of subformulae (for Boolean operators) or the number of time points in an interval (for temporal operators), the AGM aggregation operators are defined as follows.
\begin{flalign}
\begin{aligned}\label{eq:AGM_dis}
        &\mathrm{AGM}_{\lor}(r_1,\dots,r_N):= \scriptsize\begin{cases}
                -\sqrt[N]{\prod\limits_{i=1}^N(1-r_i)}+1;\\ \ \ \text{if } \forall i\in\{1,\ldots,N\}, r_i < 0\\\frac{1}{N}\sum\limits_{i=1}^N [r_i]_+;\ \ \text{otherwise }
        \end{cases}\\
    \end{aligned}
\end{flalign}
\begin{flalign}
\begin{aligned}\label{eq:AGM_con}
        &\mathrm{AGM}_{\land}(r_1,\dots,r_N):= \scriptsize \begin{cases}
            \sqrt[N]{\prod\limits_{i=1}^N(1+r_i)}-1;\\ \ \ \text{if } \forall i\in\{1,\ldots,N\}, r_i > 0\\
                \frac{1}{N}\sum\limits_{i=1}^N [r_i]_{-};\ \ \text{otherwise}
        \end{cases}\\
    \end{aligned}
\end{flalign}
When all robustness values have the same sign, AGM uses geometric mean (capturing the compound effect); otherwise, it uses arithmetic mean of the relevant parts. The STL AGM robustness semantics of signal $\mathbf{s}_{t_1,t_2}:=s_{t_1}s_{t_1+1}\dots s_{t_2}$ are defined recursively as follows \cite{cristi2019AGMstl}.
\begin{flalign}
\begin{aligned}\label{eq:eta_semantics}
     &\eta(\mathbf{s}_{t_1,t_2},\top):=1;\;\eta(\mathbf{s}_{t_1,t_2},\bot):=-1;\;\\
    &\eta(\mathbf{s}_{t},\mu):=\frac{1}{2}(h(\mathbf{s}_t)-\varsigma)\\
    &\eta(\mathbf{s}_{t_1,t_2},\lor_{i=1}^m\phi_i):=\mathrm{AGM}_{\lor}(\eta(\mathbf{s}_{t_1,t_2},\phi_1)\dots,\eta(\mathbf{s}_{t_1,t_2},\phi_m))\\
    &\eta(\mathbf{s}_{t_1,t_2},\wedge_{i=1}^m\phi_i):=\mathrm{AGM}_{\wedge}(\eta(\mathbf{s}_{t_1,t_2},\phi_1)\dots,\eta(\mathbf{s}_{t_1,t_2},\phi_m))\\
    &\eta(\mathbf{s}_{t_1,t_2},\mathbf{G}{[a,b]}\phi):= \mathrm{AGM}_{\wedge}({\eta(\mathbf{s}_{t_1,t},\phi) | t\in [t+a, t+b]}) \\
&\eta(\mathbf{s}_{t_1,t_2},\mathbf{F}{[a,b]}\phi) := \mathrm{AGM}_{\lor}({\eta(\mathbf{s}_{t_1,t},\phi | t\in [t+a, t+b]})
\end{aligned}
\end{flalign}
Values $\eta(\mathbf{s}_{t_1,t_2}, \phi) > 0$ (respectively $< 0$) indicate satisfaction (respectively dissatisfaction) $\phi$, with magnitude reflecting the strength of satisfaction. Satisfaction is indifferent when $\eta(\mathbf{s}_{t_1,t_2}, \phi)  =0$ .

\begin{example}[Illustrative Example]\label{ex:AGM_benifit_prelims}
Consider a simple 2D navigation task where a robot must visit a target region while avoiding an obstacle. Let $s_t = (x_t, y_t)$ represent the robot's position at time $t$, and consider the specification:
\begin{equation*}
\phi_{\text{nav}} = \mathbf{F}_{[0,10]}(\mu_{\text{target}}) \wedge \mathbf{G}_{[0,10]}(\mu_{\text{safe}})
\end{equation*}
where $\mu_{\text{target}} := (x-5)^2 + (y-5)^2 \leq 1$ (reach a goal region) and $\mu_{\text{safe}} := (x-3)^2 + (y-3)^2 \geq 4$ (avoid an obstacle).

For a trajectory that briefly passes near the obstacle boundary at $t=3$ (with low robustness $\approx 0.1$) but achieves high robustness in the target region at $t=8$ (robustness $\approx 0.9$), traditional min-max robustness would report the overall robustness as $\min(0.1, 0.9) = 0.1$, focusing solely on the critical constraint. In contrast, AGM robustness aggregates across all time points using arithmetic and geometric means, yielding a value closer to $0.5$ that reflects the balanced satisfaction across the entire trajectory. This holistic evaluation provides smoother guidance for optimization-based planners, as small improvements anywhere along the trajectory contribute to the overall robustness rather than being ignored unless they affect the critical constraint.    
\end{example}
\section{Problem Formulation and Approach}\label{sec:problem_formulation}
\textbf{System Model.} Consider a robot with state $\boldsymbol{q} \in \mathcal{Q} \subset \mathbb{R}^n$, where $\mathcal{Q}$ is the state space\footnote{We consider the state space in our motion planning problem.}. Let the robot dynamics be modeled as the following discrete-time system,
\begin{equation}\label{eq:system}
\begin{aligned}
\boldsymbol{q}_{k+1} = f(\boldsymbol{q}_k, \boldsymbol{u}_k),
\end{aligned}
\end{equation}
where $\boldsymbol{u}_k \in \mathcal{U} \subset \mathbb{R}^m$ is the control input at time step $k$, $\mathcal{U}$ is the allowable control set, and $f: \mathcal{Q} \times \mathcal{U} \rightarrow \mathcal{Q}$ is assumed to be a Lipschitz continuous function in its arguments; i.e., there is a constant $L\geq 0$ such that for all $(\boldsymbol{q}_1,\boldsymbol{u}_1),(\boldsymbol{q}_2,\boldsymbol{u}_2)\in\mathcal{Q}\times\mathcal{U}$ the following holds:
\begin{equation}\label{eq:Lips}
    ||f(\boldsymbol{q}_1, \boldsymbol{u}_1)-f(\boldsymbol{q}_2, \boldsymbol{u}_2)||_2\leq L\sqrt{||\boldsymbol{q}_1-\boldsymbol{q}_2||^2_2+||\boldsymbol{u}_1-\boldsymbol{u}_2||^2_2}
\end{equation}
where $||\cdot||_2$ is the Euclidean norm. 

\textbf{Trajectories and Specifications.} In this work, the signal $s_t$ corresponds to the robot state $\mathbf{q}_t$, so STL specifications are evaluated directly over state trajectories $\mathbf{q}_{0,T}$.For a given time horizon $T \in \mathbb{Z}_{>0}$, a control sequence $\mathbf{u}_{0,T-1} := {\boldsymbol{u}_0\boldsymbol{u}_1\ldots \boldsymbol{u}_{T-1}}$ generates state trajectory $\mathbf{q}_{0,T} = {\boldsymbol{q}_0\boldsymbol{q}_1\ldots \boldsymbol{q}_T}$ satisfies system (\ref{eq:system}).
We denote the control-trajectory pair as $\varphi = (\mathbf{u}_{0,T-1}, \mathbf{q}_{0,T})$.  Given an STL formula $\phi$ specifying temporal requirements on $\mathbf{q}_{0,T}$, we evaluate satisfaction using AGM robustness $\eta(\mathbf{q}_{0,T}, \phi) \in [-1, 1]$ from (\ref{eq:eta_semantics}), which provides comprehensive assessment across all time points and subformulae.

For an STL formula $\phi
$, its time horizon $||\phi||
$ is defined recursively as:
\begin{equation}\label{eq:STL_||phi||}
\|\phi\| :=
\begin{cases}
0, & \text{if } \phi = \mu,\\
\max(\|\phi_1\|,\|\phi_2\|), 
& \text{if } \phi \in \{\phi_1\land\phi_2,\ \phi_1\lor\phi_2\},\\
\|\phi_1\|, & \text{if } \phi = \neg \phi_1,\\
b+\|\phi_1\|, 
& \text{if } \phi \in \{\mathbf{G}_{[a,b]}\phi_1,\ \mathbf{F}_{[a,b]}\phi_1\}.
\end{cases}
\end{equation}


\textbf{Problem Statement.} Given initial state $\boldsymbol{q}_{\mathrm{init}}$ and STL formula $\phi$, we define the set $\mathcal{G}:= \{\varphi \mid \boldsymbol{q}_0 = \boldsymbol{q}_{\mathrm{init}},\; \eta(\mathbf{q}_{0,T-1}, \phi) > 0,\; T \leq ||\phi||\}$ containing control-trajectory pairs with positive AGM robustness.


\begin{problem}[Robust Planning Problem (RPP)]\label{pr:opt_control_problem}
Given a robot with system dynamics (\ref{eq:system}), an initial state $\boldsymbol{q}_0 \in \mathcal{Q}$, and an STL specification $\phi$, find $\varphi^\ast \in \mathcal{G}$ that maximizes the AGM robustness i.e., $\varphi^\ast = \argmax\limits_{\varphi \in \mathcal{G},\;T \in \mathbb{Z}_{>0}} \eta(\mathbf{q}_{0,T-1}, \phi)$. 
\end{problem}

RPP can be viewed as a search in the set $\mathcal{G}$, seeking control sequences that maximize specification satisfaction strength. The AGM robustness objective creates a non-convex optimization landscape, but one that is significantly smoother than traditional min-max robustness due to its holistic aggregation across time points and subformulae (see Example \ref{ex:AGM_benifit_prelims}).

\textbf{Approach Overview.} We address RPP through RRT$^\eta$, which integrates AGM robustness into the RRT$^\ast$ framework. Our algorithm adapts the Direction of Increasing Satisfaction (DIS) from STL-RRT$^\ast$\cite{CristiKaraman17_STL_RRTstar} to leverage AGM's gradient information across all time points and subformulae. Additionally, we develop interval semantics for AGM robustness to reason about partial trajectories during tree construction, enabling informed exploration even before complete paths are formed. These innovations maintain RRT$^\ast$'s asymptotic optimality while providing more consistent guidance through the non-convex landscape.

\section{Online Robustness Evaluation for Sampling-Based Planning}\label{sec:AGMMonitor}
A critical challenge in our sampling-based planning approach is that RRT$^\eta$ builds trajectories incrementally, producing partial trajectories whose time horizons are typically less than $||\phi||$. To effectively guide tree expansion with partial trajectories, we need a method to evaluate the potential robustness of incomplete trajectories.

We adopt the interval-based approach from \cite{Seshia17_STLmtrng_robustSatInterv,ahmad2023TWTLrobustness} to construct AGM robustness $[\eta]_{\mathbf{s}_{t^\prime},\phi}\in\mathcal{I}$, where $\mathcal{I}:=\{[\underline{\eta},\overline{\eta}]\;|\underline{\eta},\overline{\eta}\in [-1,1],\underline{\eta}\leq\overline{\eta}\}$, that bound all possible robustness values for any completion of a partial trajectory at time $t^\prime$. Our AGM robustness interval specializes the TWTL AGM robustness interval from \cite{ahmad2023TWTLrobustness} for STL specifications in motion planning.

\begin{definition}[Arithmetics on interval semantics] 
    Consider the set of intervals $\{I_i\}^N_{i=1}\subseteq\mathcal{I}$. We define following arithmetics over $\mathbf{I}:=\{I_i\}^N_{i=1}$
    \begin{equation}
        \begin{aligned}\label{eq:AGM_interval_Arithm}
            \mathrm{\mathbf{AGM}}_{\lor}(\mathbf{I}):= [\mathrm{AGM}_{\lor}(\underline{I}_1,\dots,\underline{I}_N),\mathrm{AGM}_{\lor}(\overline{I}_1,\dots,\overline{I}_N)],\\
            \mathrm{\mathbf{AGM}}_{\land}(\mathbf{I}):= [\mathrm{AGM}_{\land}(\underline{I}_1,\dots,\underline{I}_N),\mathrm{AGM}_{\land}(\overline{I}_1,\dots,\overline{I}_N)].
        \end{aligned}
    \end{equation}
\end{definition}

These intervals enable RRT$^\eta$, as we detail in the next section,  to make informed decisions during tree expansion even with incomplete trajectories.


 We introduce an incremental, recursive AGM robustness monitor function $\texttt{IRTM}: \Phi\times\mathcal{I} \times\mathcal{I} \times \mathbb{R} \times \mathbb{Z}_{>0} \times \mathbb{Z}_{>0} \rightarrow \mathcal{I}$, implemented in Algorithm \ref{alg:EfficientIncrem_IRTM}. This function computes the updated robustness interval at time $t^\prime$ given the previous interval $[\eta]_{\mathbf{s}_{t^\prime-1},\phi}$, the current trajectory observation $\mathbf{s}_{t^\prime}$, and the starting time $t_s
$. Rather than recomputing the entire interval at each time step, $\texttt{IRTM}$ uses incremental modification functions (Definition \ref{def:modfAGM}) to efficiently update only components affected by new observations.

\begin{definition}[Incremental Modification Functions]\label{def:modfAGM}
To efficiently compute the incremental AGM robustness interval, we define helper functions $\texttt{mdf\_AGM}_{\lor},\texttt{mdf\_AGM}_{\land}:\;[-1,1]\times\mathbb{Z}_{>0}\times[-1,1]\to[-1,1]$ that update existing AGM robustness values when new observations become available. When monitoring AGM robustness for a formula with $N$ time points or subformulae, these functions efficiently incorporate the $N$th observation by applying arithmetic or geometric mean operations incrementally.

For disjunction ($\lor$), we define:
\begin{equation}\label{eq:mod_agmor}
    \texttt{mdf\_AGM}_{\lor}(\eta,N,\eta^\prime):= \scriptsize\begin{cases}
                -\sqrt[N]{(1-\eta)^{N-1}\cdot(1-\eta^\prime)}+1;\\ \ \ \text{if } \eta < 0 \;\land\;\eta^\prime < 0\\
                \frac{[\eta^\prime]_{+}}{N};\qquad \text{if } \eta <0 \land \eta^\prime>0
                \\\frac{(N\cdot\eta-[\eta]_{+})+[\eta^\prime]_+}{N};\ \ \text{otherwise }
        \end{cases}
\end{equation}

For conjunction ($\land$), we define:
\begin{equation}\label{eq:mod_agmand}
    \texttt{mdf\_AGM}_{\land}(\eta,N,\eta^\prime):= \scriptsize \begin{cases}
           \sqrt[N]{(1+\eta)^{N-1}\cdot(1+\eta^\prime)}-1;\\ \ \ \text{if } \eta>0 \land \eta^\prime >0\\
                \frac{[\eta^\prime]_{-}}{N};\qquad \text{if } \eta >0 \land \eta^\prime<0
            \\
                \frac{(N\cdot\eta+[\eta]_{-})+[\eta^\prime]_{-}}{N};\ \ \text{otherwise}
        \end{cases}
\end{equation}
\end{definition}

\begin{lemma}[Correctness of Incremental Modification Functions]\label{lemma:mdf_correctness}
Let $\eta = \mathrm{AGM}_{\circ}(r_1, \ldots, r_{N-1})$ be the AGM robustness computed from $N-1$ values, where $\circ \in \{\lor, \land\}$. For any new observation $\eta'$, the incremental modification functions produce identical results to full AGM recomputation:
\begin{equation}
\texttt{mdf\_AGM}_{\circ}(\eta, N, \eta') = \mathrm{AGM}_{\circ}(r_1, \ldots, r_{N-1}, \eta')
\end{equation}
\end{lemma}

\begin{proof}[Sketch]
The proof proceeds by case analysis on the signs of $\eta$ and $\eta'$, corresponding to the three cases in equations (\ref{eq:mod_agmor}) and (\ref{eq:mod_agmand}).

\textbf{Geometric Mean Case ($\eta < 0, \eta' < 0$ for $\lor$; $\eta > 0, \eta' > 0$ for $\land$):} 
Since $\eta$ resulted from geometric mean aggregation, we have $(1 \mp \eta)^{N-1} = \prod_{i=1}^{N-1}(1 \mp r_i)$ (where $\mp$ is $-$ for disjunction, $+$ for conjunction). Incorporating the new value $\eta'$ yields $\mathrm{AGM}_{\circ}(r_1, \ldots, r_{N-1}, \eta') = \pm\sqrt[N]{(1 \mp \eta)^{N-1}(1 \mp \eta')} \pm 1$, which exactly matches the first case of each modification function.

\textbf{Transition Case ($\eta < 0, \eta' > 0$ for $\lor$; $\eta > 0, \eta' < 0$ for $\land$):} 
When all previous values had one sign and the new value has the opposite sign, the AGM switches from geometric to arithmetic mean. Since all previous contributions were zero in the arithmetic formulation ($[r_i]_{\pm} = 0$ when signs differ from required), only the new value contributes: $\mathrm{AGM}_{\circ}(\ldots, \eta') = \frac{[\eta']_{\pm}}{N}$.

\textbf{Arithmetic Mean Case (otherwise):} 
When at least one previous value had the appropriate sign for arithmetic aggregation, we have $(N-1)\eta = \sum_{i=1}^{N-1}[r_i]_{\pm}$. However, $[\eta]_{\pm}$ captures the result rather than the sum of inputs, so the correct update is $\frac{(N \cdot \eta - [\eta]_{\pm}) + [\eta']_{\pm}}{N}$.

The full proof with detailed algebraic manipulations is provided in Appendix \ref{apndx:mdf_correctness}.
\end{proof}

The AGM robustness monitor efficiently tracks interval evolution as partial signals extend. Algorithm \ref{alg:EfficientIncrem_IRTM} handles general STL operators, while Algorithm \ref{alg:IRTM_T} addresses temporal operators requiring complex interval reasoning. The incremental approach avoids redundant computations while maintaining sound bounds, providing critical guidance for tree growth by identifying promising directions likely to improve specification satisfaction.


\subsection{Theoretical Analysis}\label{subsec:analysis}
We begin by establishing the soundness of our AGM robustness interval calculation. 
\begin{definition}[Prefix, Completions]
    Consider the time horizon $||\phi||$ and trajectories $\mathbf{s}_{t_1,t^\prime}$ and $\mathbf{s}_{t_1,t_2}$, where $t^\prime<||\phi||$ and $t_2\geq||\phi||$. We denote $\mathbf{s}_{t_1,t^\prime}$ as a prefix of $\mathbf{s}_{t_1,t_2}$ if $\forall t\in[t_1,t^\prime],\mathbf{s}_{t_1,t_2}(t)=\mathbf{s}_{t_1,t^\prime}(t)$. The set of all possible completions of a prefix $\mathbf{s}_{t_1,t^\prime}$ is $\mathfrak{C}:=\{\mathbf{s}_{t_1,t_2} \mid \mathbf{s}_{t_1,t^\prime} \text{ is a prefix of } \mathbf{s}_{t_1,t_2}\}$ 
\end{definition}
\begin{lemma}[AGM Robustness Interval Soundness]\label{lemma:rosiSoundness}
 For any STL formula $\phi$, the valuation $[\eta]_{\mathbf{s}_{t^\prime},\phi}\gets\texttt{IRTM}(\phi,[\eta]_{\mathbf{s}_{t^\prime},\phi},\emptyset,\mathbf{s}_{t^\prime},t^\prime,t_{0})$ defines the AGM robust satisfaction interval for the formula $\phi$ over the partial trajectory $\mathbf{s}_{t_0,t^\prime}$ at time $t_0$. For any completion $\mathbf{s}\in \mathfrak{C}(\mathbf{s}_{t_0,t^\prime})$, the AGM robustness value is contained in the computed interval, i.e., $\eta(\mathbf{s}, \phi) \in [\eta]_{\mathbf{s}_{t^\prime},\phi}$.
\end{lemma}

\begin{proof}[Sketch]
The proof follows by structural induction over STL formulae, leveraging properties of arithmetic and geometric means in the AGM robustness definition. The full proof is provided in Appendix \ref{apndx:ProofofROSIsoundness}.
\end{proof}

Building on soundness, we establish the relationship between robustness intervals as partial trajectories grow:

\begin{theorem}[AGM Robustness Interval Chain Inclusion]\label{thm:rosiChainInclusion}
    Given a partial signal $\mathbf{s}_{t_0,t^\prime}$, STL formula $\phi$, and $t \leq t^\prime$, we have $[\eta]_{\mathbf{s}_{t^\prime},\phi} \subseteq [\eta]_{\mathbf{s}_{t},\phi}$.
\end{theorem}

\begin{proof}[Sketch]
Since $\mathbf{s}_{t_0,t^\prime}$ contains more information than $\mathbf{s}_{t_0,t}$ when $t^\prime > t$, the set of possible completions reduces to $\mathfrak{C}(\mathbf{s}_{t_0,t^\prime}) \subseteq \mathfrak{C}(\mathbf{s}_{t_0,t})$. By Lemma \ref{lemma:rosiSoundness}, the AGM robustness interval contains all possible robustness values of completions. Fewer possible completions yield narrower robustness value ranges, resulting in the inclusion property. The full proof is in Appendix \ref{apndx:ChainInclusion}.
\end{proof}

A natural consequence is interval convergence as trajectories extend:

\begin{corollary}[AGM Robustness Interval Convergence]\label{corollary:AGM_convgnc} As the partial trajectory extends through additional time steps, AGM robustness bounds shrink. When the trajectory horizon equals or exceeds the formula horizon $||\phi||$, the interval converges to a singleton containing the exact AGM robustness value: $[\eta]_{\mathbf{s}_{t^\prime},\phi} = \{\eta(\mathbf{s}_{t_0,t^\prime}, \phi,t_0)\}$ for $t^\prime \geq ||\phi||$.
\end{corollary}
\begin{proof}[Sketch]
    This follows directly from the chain inclusion property. Once $t^\prime \geq ||\phi||$, the formula can be fully evaluated on the available trajectory, leaving no uncertainty about the robustness value.
\end{proof}

These properties ensure that our AGM robustness monitor provides sound, increasingly precise estimates of specification satisfaction as trajectories evolve, making it ideal for incremental motion planning.

\begin{algorithm}[t]\scriptsize
    \KwIn{ $\phi$, $[\eta]$, $[\eta]_{\mathrm{aux}}$, $\mathbf{s}_{t^\prime}$, $t^\prime$,$t_s$}
    \KwOut{$[\eta]'(\phi)$}
    \DontPrintSemicolon
    \BlankLine
    \If{$[\eta]_{\mathrm{aux}} \neq \emptyset$}{
        \Return $[\eta]_{\mathrm{\phi}}$
    }
    \ElseIf{$\phi = \mu$}{
        $[\eta]_{\mathrm{aux}} \gets [h(\mathbf{s}_{t^\prime}),h(\mathbf{s}_{t^\prime})]$
    }
    \ElseIf{$\phi \in \left\{ \bigwedge \phi_i, \bigvee \phi_i \right\}$}{
        $\mathfrak{E}\gets \emptyset$\;
        \ForEach{$i \in \{1,\ldots,N\}$}{
            $\mathfrak{E} \gets \mathfrak{E} \cup \{\texttt{IRTM}(\phi_i, [\eta], [\eta]_{\mathrm{aux}}, \mathbf{s}_{t^\prime},t^\prime,t_s)\}$
        }
        \eIf{AND}{
            $[\eta]_{\mathrm{aux}}(\phi) \gets \mathrm{\mathbf{AGM}}_{\land}(\mathfrak{E})$
        }{
            $[\eta]_{\mathrm{aux}}(\phi) \gets \mathrm{\mathbf{AGM}}_{\lor}(\mathfrak{E})$
        }
    }
    \ElseIf{$\phi = \textbf{G}_{[a,b]}\; \phi_1$}{
        $[\eta]_{\mathrm{aux}} \gets \texttt{IRTM}(\phi_1, [\eta],\emptyset, \mathbf{s}_{t^\prime},t^\prime,t_s)$\;
        $[\eta]_{\mathrm{aux}} \gets \texttt{IRTM}_{\textbf{T}}([\eta], [\eta]_{\mathrm{aux}}, t^\prime,t_s,a,b,||\phi||,\textbf{G})$ \Comment{Algorithm \ref{alg:IRTM_T}}\label{line:alg_ITRM_G}
    }
    \ElseIf{$\phi = \textbf{F}_{[a,b]}\phi_1$}{
        $[\eta]_{\mathrm{aux}} \gets \texttt{IRTM}(\phi_1, [\eta],\emptyset, \mathbf{s}_{t^\prime},t^\prime,t_s)$\;
        $[\eta]_{\mathrm{aux}} \gets \texttt{IRTM}_{\textbf{T}}([\eta], [\eta]_{\mathrm{aux}}, t^\prime,t_s,a,b,||\phi||,\textbf{F})$\Comment{Algorithm \ref{alg:IRTM_T}}\label{line:alg_ITRM_F}
    }
    \Return{$[\eta]_{\mathrm{aux}}(\phi)$}
    \caption{ $\texttt{IRTM}(\phi,[\eta],[\eta]_{\mathrm{aux}}, \mathbf{s}_{t^\prime}, t^\prime, t_s)$}
 \label{alg:EfficientIncrem_IRTM}
\end{algorithm}

\setlength{\textfloatsep}{15pt}
\begin{algorithm}[t]\footnotesize
    \textbf{Input:} $[\underline{\eta},\overline{\eta}]\gets [\eta]$, $[\underline{\eta}^\prime,\overline{\eta}^\prime]\gets[\eta]_{\mathrm{aux}}$, $t_s$, $t^\prime$, $a$, $b$\\
    \textbf{Output:} $[\underline{\eta},\overline{\eta}]$\\
        \If{$t^\prime < t_s+a$}
            {
            \Return{$\emptyset$}
            }
        \If{$\underline{\eta}=\overline{\eta}$}
        {\Return{$[\underline{\eta},\overline{\eta}]$}}
        \If{$[\eta]=\emptyset$}
            { $\mathfrak{I}_{-}\gets\{\underline{\eta}^\prime,-1,\dots,-1\}$, $\mathfrak{I}_{+}\gets\{\overline{\eta}^\prime,1,\dots,1\}$; where $|\mathfrak{I}_{-}|=|\mathfrak{I}_{+}|=||\phi||$ \\
            \If{\texttt{type}=\textbf{G}}{$\underline{\eta}^\prime\gets\mathrm{AGM}_{\land}(\mathfrak{I}_{-})$\\
            $\overline{\eta}^\prime\gets\mathrm{AGM}_{\land}(\mathfrak{I}_{+})$}
            
            \If{\texttt{type}=\textbf{F}}{$\underline{\eta}^\prime\gets\mathrm{AGM}_{\lor}(\mathfrak{I}_{-})$\\$\overline{\eta}^\prime\gets\mathrm{AGM}_{\lor}(\mathfrak{I}_{+})$}\Return{$[\eta^\prime,\eta^\prime]$}}
            \eIf{$t^\prime \geq t_s+b$}
        {
        \If{\texttt{type}=\textbf{G}}{$\eta^\prime\gets\texttt{mdf\_AGM}_{\land}(\overline{\eta},N,\overline{\eta}^\prime)$}
        \If{\texttt{type}=\textbf{F}}{$\eta^\prime\gets\texttt{mdf\_AGM}_{\lor}(\overline{\eta},N,\overline{\eta}^\prime)$}
        \Return{$[\eta^\prime,\eta^\prime]$}}
        {
        \If{\texttt{type}=\textbf{G}}{$\underline{\eta}^\prime\gets\texttt{mdf\_AGM}_{\land}(\underline{\eta},N,\underline{\eta}^\prime)$\\
        $\overline{\eta}^\prime\gets\texttt{mdf\_AGM}_{\land}(\overline{\eta},N,\overline{\eta}^\prime)$}
        
         \If{\texttt{type}=\textbf{F}}{$\underline{\eta}^\prime\gets\texttt{mdf\_AGM}_{\lor}   (\underline{\eta},N,\underline{\eta}^\prime)$
        
        $\overline{\eta}^\prime\gets\texttt{mdf\_AGM}_{\lor}(\overline{\eta},N,\overline{\eta}^\prime)$}
        \Return{$[\underline{\eta}^\prime ,\overline{\eta}^\prime]$}
        }
	\caption{$\texttt{IRTM}_{\textbf{T}}([\eta],[\eta]_{\mathrm{aux}},t_s,t^\prime,a,b,||\phi||,\texttt{type})$}
 \label{alg:IRTM_T}
\end{algorithm}
\setlength{\textfloatsep}{14pt}

\subsection{Complexity Analysis}\label{subsec:IRTM_complxty}

For an STL formula $\phi$ with $|\phi|$ operators and predicates, each call to $\texttt{IRTM}$ (Algorithm \ref{alg:EfficientIncrem_IRTM}) processes a single new observation by performing a recursive traversal of the formula's abstract syntax tree. The algorithm visits each node exactly once: predicates evaluate $h(\mathbf{s}_{t^\prime})$ in $O(1)$ time (Line 4), Boolean operators with $m$ subformulae recursively evaluate each subformula then aggregate results in $O(m)$ time (Lines 5-12), and temporal operators invoke $\texttt{IRTM}_{\mathbf{T}}$ after evaluating the subformula (Lines 13-18). The subroutine $\texttt{IRTM}_{\mathbf{T}}$ (Algorithm \ref{alg:IRTM_T}) executes in $O(1)$ time, performing only conditional checks and calls to $\texttt{mdf\_AGM}_{\land}$ or $\texttt{mdf\_AGM}_{\lor}$, which require constant-time arithmetic by Lemma \ref{lemma:mdf_correctness}. Thus, each $\texttt{IRTM}$ invocation has complexity $O(|\phi|)$.

In contrast, a non-incremental implementation that recomputes the entire robustness interval from scratch at each time step must evaluate the formula over the complete partial trajectory $\mathbf{s}_{t_0,t^\prime}$, requiring $O(|\mathbf{s}_{t_0,t^\prime}| \cdot |\phi|)$ operations per update. Over a complete trajectory of length $|\mathbf{s}|$, the incremental approach requires $O(|\mathbf{s}| \cdot |\phi|)$ total operations while non-incremental methods require $O(|\mathbf{s}|^2 \cdot |\phi|)$, providing a linear speedup factor of $O(|\mathbf{s}|)$. For typical planning scenarios with $|\mathbf{s}| \approx 20$-$50$ and $|\phi| \approx 10$-$20$, this reduces monitoring overhead from thousands to hundreds of operations per trajectory evaluation.

\section{RRT$^\eta$ Motion Planning Algorithm}\label{sec:RRTeta}
In this section, we detail the algorithmic formulation of RRT$^\eta$, which extends the sampling-based planning approach of STL-RRT$^\ast$ \cite{CristiKaraman17_STL_RRTstar} to incorporate AGM robustness for STL specifications.

\subsection{Overview}\label{subsec:RRTetaOverview}
RRT$^\eta$ enhances sampling-based motion planning through integration of AGM robustness for STL specifications. The algorithm operates through two key components working together to efficiently find robust solutions.

The main planning loop (Algorithm \ref{alg:RRTeta}) constructs a tree of states that progressively explores the state space while guided by AGM robustness considerations. The algorithm balances exploration and exploitation by sampling with STL-guided bias and steering toward configurations maximizing AGM robustness. For each potential connection, the algorithm computes the Direction of Increasing AGM Satisfaction (DIAS), which guides steering toward regions with higher specification satisfaction.

Supporting this process, an update procedure (Algorithm \ref{alg:updateAGM}) maintains and propagates AGM robustness intervals throughout the tree. This procedure incrementally processes trajectory segments, evaluating them against the STL specification using our $\texttt{IRTM}$ function. Importantly, this allows informed decisions about node additions and rewiring operations based on partial trajectories before complete paths are established.

The synergy between efficient tree construction and incremental robustness evaluation enables RRT$^\eta$ to identify solutions that not only satisfy STL specifications but do so with maximum robustness across all aspects of the specification.

The synergy between efficient tree construction and incremental robustness evaluation enables RRT$^\eta$ to identify solutions that not only satisfy STL specifications but do so with maximum robustness across all specification aspects.

\subsection{Direction of Increasing AGM Satisfaction}\label{subsec:DIAS}

Our RRT$^\eta$ algorithm leverages a modified version of the DIS concept introduced by Vasile et al. \cite{CristiKaraman17_STL_RRTstar}. The key innovation is integration of AGM robustness intervals $[\eta]_{\mathbf{s}_{t^\prime},\phi}$ to guide tree expansion and steering processes.

The DIS provides gradient-like information guiding exploration toward regions that improve specification satisfaction. We define the Direction of Increasing AGM Satisfaction (DIAS) as a function $\chi_\eta:\mathbb{R}^n\times \Phi\to\mathbb{R}^n$ that takes a trajectory, an STL formula, and a time point, then computes a vector pointing in the direction of increasing AGM robustness.

\subsubsection{Base Cases and Temporal Operators}

For base cases and temporal operators, the DIAS computation follows the established structure. Considering state $\boldsymbol{q}_{t}$ of system (\ref{eq:system}), we define:

\begin{flalign}
\begin{aligned}\label{eq:chi_semantics}
     &\chi_\eta(\boldsymbol{q}_{t}, \top) := \mathbf{0}_n\\
     &\chi_\eta(\boldsymbol{q}_{t}, \mu) :=
\begin{cases}
\nabla_{\boldsymbol{q}}\eta(\boldsymbol{q}_{t}, \mu)^{\mathrm{T}} \cdot \mathbf{J}_f(\boldsymbol{q}_t, \boldsymbol{u}_t);\\
\quad \text{if } \nabla_{\boldsymbol{q}}\eta(\boldsymbol{q}_{t}, \mu)^{\mathrm{T}} \cdot \mathbf{J}_f(\boldsymbol{q}_t, \boldsymbol{u}_t)\cdot \boldsymbol{u}_t > 0\\
\mathbf{0}_n; \quad\text{otherwise}
\end{cases}\\
    &\chi_\eta(\boldsymbol{q}_{t}, \mathbf{G}_{[a,b]}\phi_1) = \chi_\eta(\boldsymbol{q}_{t}, \mathbf{F}_{[a,b]}\phi_1) := \chi_\eta(\boldsymbol{q}_{t}, \phi_1)
\end{aligned}
\end{flalign}
where $\mathbf{J}_f(\boldsymbol{q}_t, \boldsymbol{u}_t)= [\frac{\partial f}{\partial q_1}\dots\frac{\partial f}{\partial q_n}]$ is the Jacobian matrix of dynamics function $f$ evaluated at state $\boldsymbol{q}_t$ and control input $\boldsymbol{u}_t$. The dot product in $\chi_{\eta}(\boldsymbol{q}_t,\mu)$ indicates whether system dynamics will naturally increase the predicate's satisfaction measure at the current state.

\begin{example}[DIAS of Fixed Control Action]
    Consider a 2D discrete-time dynamical system with state $q = [x, y]^\top \in \mathbb{R}^2$ and control input $u \in \mathbb{R}^2$. The system dynamics are given by $q_{k+1} = q_k + u$. The Jacobian with respect to the state is $\mathbf{J}_f  = \mathbf{I}_2$, where $\mathbf{I}_2\in\mathbb{R}^{2\times 2}$ is the identity matrix.
    
    Let the spatial predicate be $\mu = \{q \mid \|q - c\|^2 \leq r^2\}$ with center $c = [3.5, 3.5]^\top$ and radius $r = 1$. The robustness is $\eta(q, \mu) = r - \|q - c\|$, normalized to range $[-1, 1]$. The gradient is $        \nabla_q \eta(q,\mu) = -\frac{q - c}{\|q - c\|},$ which points toward the center $c$ with unit magnitude.
    
    For a fixed control input $u = [0.5, 0.5]^\top$, the DIAS is $\chi_\eta(q,\mu) = \nabla\eta(q,\mu)^\top$.
    
    However, $\chi_\eta(q,\mu)$ is nonzero only where the condition $\nabla\eta(q,\mu)^\top \cdot (q_{k+1} - q_k) > 0$ is satisfied, i.e., where the state transition aligns with the direction of increasing satisfaction. Figure~\ref{fig:dias_illustration} visualizes this: the DIAS field (d) is active only in regions where the gradient field (b) and the state transition direction (c) point in similar directions, ensuring that the control action moves the system toward higher robustness.
\end{example}

\begin{figure}
    \centering
    \includegraphics[width=0.8\columnwidth]{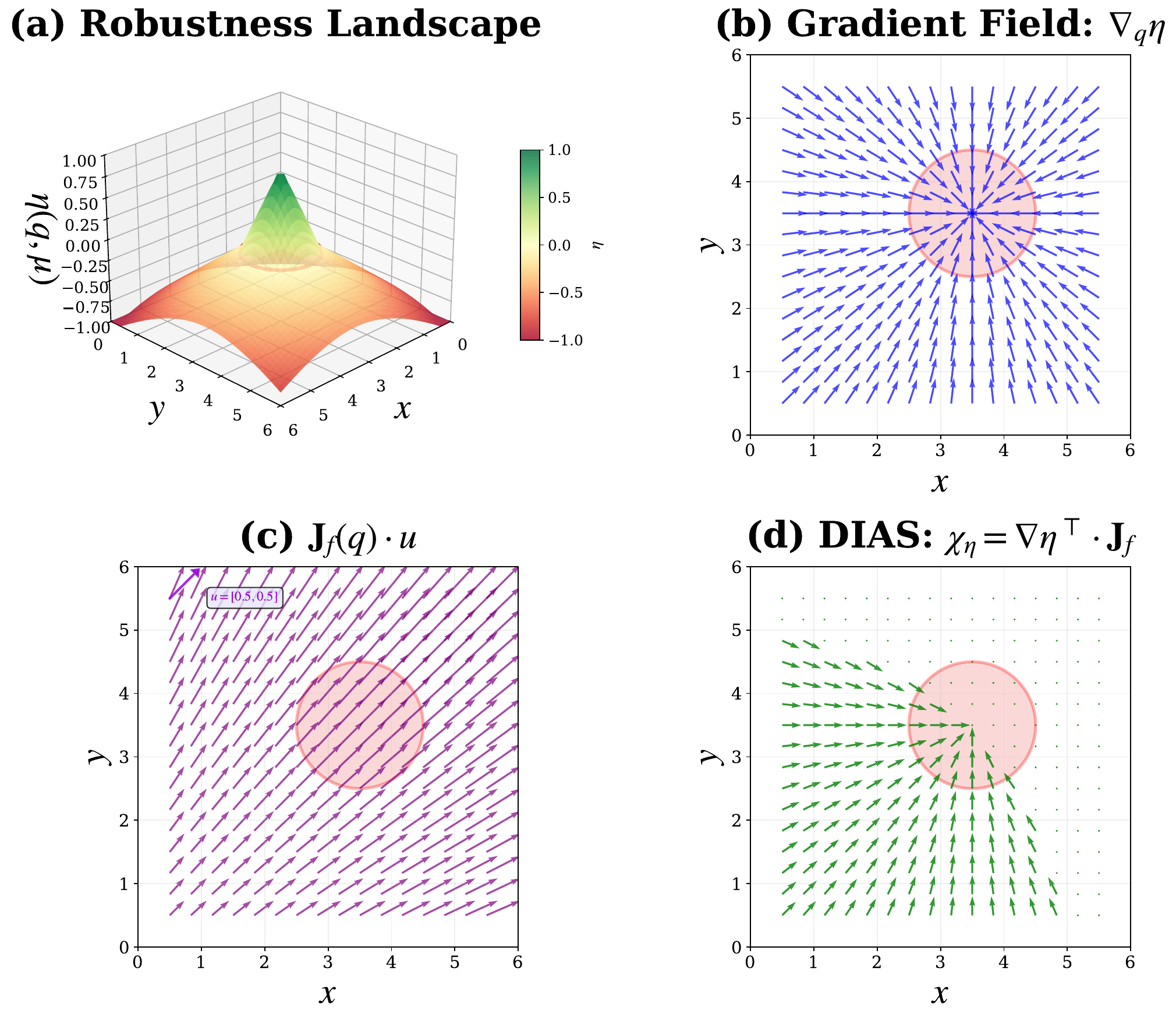}
     \caption{Illustration of the Direction of Increasing Satisfaction (DIAS) for a spatial predicate $\mu = \{q \mid \|q - c\|^2 \leq r^2\}$ with center $c = [3.5, 3.5]^\top$ and radius $r = 1$. (a) The robustness landscape $\eta(q, \mu)$ showing the satisfaction region (red circle) where $\eta > 0$. (b) The gradient field $\nabla_q \eta(q,\mu)$ pointing toward the direction of increasing satisfaction. (c) The state transition field $\mathbf{J}_f(q) \cdot u$ resulting from the control input $u = [0.5, 0.5]^\top$ and discrete-time dynamics with Jacobian $\mathbf{J}_f = \mathbf{I}_2$. (d) The DIAS field $\chi_\eta(q,\mu) = \nabla\eta(q,\mu)^\top \cdot \mathbf{J}_f$, which is nonzero only in regions where $\nabla\eta(q,\mu)^\top \cdot (q_{k+1} - q_k) > 0$ (condition satisfied).}
    \label{fig:dias_illustration}

\end{figure}

\subsubsection{Boolean Operators: The Composition Challenge}

For Boolean operators $\phi_1 \circ \phi_2$ where $\circ \in \{\wedge, \vee\}$, we must compose DIAS vectors from multiple subformulae. This composition is critical as it determines how the planner balances competing objectives. We define:

\begin{equation}\label{eq:chi_boolean}
\begin{aligned}
    \chi_\eta(\boldsymbol{q}_t, \phi_1 \circ \phi_2) :=& \texttt{compose}(\chi_{\eta}(\boldsymbol{q}_t,\phi_{1}), \\&\chi_{\eta}(\boldsymbol{q}_t,\phi_{2}), [\eta]_{\boldsymbol{q}_t,\phi_1}, [\eta]_{q_t,\phi_2})
\end{aligned}
\end{equation}

The choice of composition function significantly impacts planning performance. We present two approaches: a stochastic baseline adapted from prior work, and our principled FPL-based method.

\subsection{Composition Approaches for Boolean Operators}\label{subsec:composition}

\subsubsection{Stochastic Composition}\label{subsubsec:stochastic}

We first adapt the stochastic composition approach from Vasile et al. \cite{CristiKaraman17_STL_RRTstar} to work with AGM robustness intervals. While their approach was designed for traditional min-max robustness, we extend it to handle AGM intervals by defining:

\begin{equation}\label{eq:stochastic_compose}
\texttt{compose}_{\text{stoch}}(\chi_1, \chi_2, [\eta]_1, [\eta]_2) := \texttt{blend}_\eta(\chi_c, \chi_{\neg c})
\end{equation}

where $(c, \neg c) = \texttt{choose}_{\eta}([\eta]_1, [\eta]_2)$ with the stochastic choice function:

\begin{equation}\label{eq:choose}
\begin{aligned}
&\texttt{choose}_{\eta}([\eta]_{\boldsymbol{q}_t, \phi_1},[\eta]_{\boldsymbol{q}_t, \phi_2}):=\\
&\begin{cases}
(1,2) & \text{if } \underline{\eta}_1 < \underline{\eta}_2 \land \overline{\eta}_1 < \overline{\eta}_2\\
(2,1) & \text{if } \underline{\eta}_1 > \underline{\eta}_2 \land \overline{\eta}_1 > \overline{\eta}_2\\
(1+\text{Ber}(p),2-\text{Ber}(p)) & \text{otherwise}
\end{cases}
\end{aligned}
\end{equation}

where $p = 0.5 + \frac{(\underline{\eta}_1 + \overline{\eta}_1) - (\underline{\eta}_2 + \overline{\eta}_2)}{8}$, and $[\underline{\eta}_i,\overline{\eta}_i]=[\eta]_{\boldsymbol{q}_t, \phi_i}$ for $i \in \{1,2\}$.

This stochastic choice mechanism is crucial for preserving asymptotic optimality of RRT$^\ast$. By introducing randomness when the dominance relationship between robustness intervals is ambiguous, we ensure the algorithm maintains non-zero probability of exploring all possible directions. This randomized decision-making, while biased toward more promising options, prevents the algorithm from permanently excluding any potentially optimal regions of the state space.

The geometric blending function combines the selected directions:

\begin{equation}\label{eq:blend}
\begin{aligned}
\texttt{blend}_{\eta}(\chi_c, \chi_{\neg c}) :=
\begin{cases}
\chi_c + \chi_{\neg c} & \text{if } \chi_c \perp \chi_{\neg c}\\
\chi_c & \text{otherwise}
\end{cases}
\end{aligned}
\end{equation}

The intuition behind this blend function is to maximize useful gradient information while avoiding conflicting guidance. When DIAS vectors from two subformulae are orthogonal, they provide complementary information that doesn't conflict. In such cases, the blend function adds them together, allowing the algorithm to simultaneously improve satisfaction of both subformulae. When vectors aren't orthogonal, potential conflicts arise where improving one subformula might worsen another. In these cases, the blend function prioritizes the direction chosen by the choice function, preventing oscillatory behavior or poor progress from conflicting guidance.

\subsubsection{FPL-Based Principled Composition}\label{subsubsec:FPL}

To provide more principled composition that explicitly balances competing subformulae based on their fulfillment levels, we introduce an FPL-based approach leveraging the mathematical framework from Fulfillment Priority Logic \cite{mabsout2025FPL}.

\paragraph{Fulfillment Values and Power Means.} 

FPL provides a family of aggregation operators based on power means that unify minimum and maximum operations. We derive fulfillment values $f \in [0,1]$ from AGM robustness intervals by mapping from $[-1,1]$ to $[0,1]$:

\begin{equation}\label{eq:fulfillment_mapping}
f_i = \frac{(\underline{\eta}_i + \overline{\eta}_i + 2)}{4}
\end{equation}where $[\underline{\eta}_i, \overline{\eta}_i]$ is the AGM robustness interval for subformula $\phi_i$.

The power mean operator $\mu_p$ provides a continuous family of aggregation operations:

\begin{equation}\label{eq:power_mean}
\mu_p(\vec{f}) = \left(\frac{1}{n}\sum_{i=1}^{n}f_i^p\right)^{\frac{1}{p}}
\end{equation}

The parameter $p$ controls composition behavior: $p \rightarrow -\infty$ approaches minimum (conservative/conjunctive), $p = 1$ gives arithmetic mean, and $p \rightarrow \infty$ approaches maximum (optimistic/disjunctive). This provides a principled way to interpolate between worst-case and best-case aggregation.

\paragraph{FPL-Based Direction Composition} We define a unified FPL-based composition function for DIAS vectors:

\begin{equation}\label{eq:fpl_compose}
\begin{aligned}
\texttt{compose}_{\text{FPL}}(&\{(\chi_i, [\eta]_i)\}_{i=1}^m, \circ) :=\\
&\begin{cases}
\sum\limits_{i=1}^m \chi_i & \text{if } \chi_i \perp \chi_j, \forall i \neq j\\
\sum\limits_{i=1}^m w_i \cdot \chi_i & \text{otherwise}
\end{cases}
\end{aligned}
\end{equation}where the weights $w_i$ are computed using derivatives of the power mean:

\begin{equation}\label{eq:fpl_weights}
w_i = \frac{f_i^p \cdot \frac{\partial \mu_p(\vec{f})}{\partial f_i}}{\sum_{j=1}^m f_j^p \cdot \frac{\partial \mu_p(\vec{f})}{\partial f_j}} + \alpha_i
\end{equation}

Here, $p = -1$ for conjunction ($\circ = \land$) and $p = 1$ for disjunction ($\circ = \lor$). The term $\alpha_i$ adds controlled randomness to preserve asymptotic optimality:

\begin{equation}\label{eq:randomization}
\alpha_i = \beta \cdot r_i \cdot (1 - \max_{j \neq i}|f_i - f_j|)
\end{equation}where $\beta$ is a small scaling factor (typically 0.1), $r_i \in [-1,1]$ is randomly sampled, and the term $(1 - \max_{j \neq i}|f_i - f_j|)$ ensures randomness diminishes as fulfillment differences increase.

This formulation preserves the orthogonality consideration from the stochastic blend function while providing more nuanced weighting based on fulfillment contributions. The gradient-based weights naturally prioritize less-fulfilled objectives, providing principled balancing. The randomization term maintains probabilistic completeness and asymptotic optimality of RRT$^\ast$ while producing more principled direction compositions than purely stochastic selection.

\paragraph{Minimum Fulfillment Guarantees} A key advantage of using FPL is the power mean's minimum fulfillment bound:

\begin{equation}\label{eq:min_fulfillment}
\forall p \in \mathbb{R}, \vec{f} \in [0,1]^n, \min(\vec{f}) \geq \sqrt[p]{n((\mu_p(\vec{f}))^p - 1) + 1}
\end{equation}

This bound guarantees that when a power mean outputs value $y$, every input component must have at least fulfillment $\sqrt[p]{n(y^p - 1) + 1}$. For conjunction operations, more negative $p$ values provide stronger guarantees on minimum fulfillment, ensuring all subformulae are adequately satisfied rather than focusing predominantly on the most critical one.

\subsection{Algorithm Details}\label{subsec:RRTeta_details}

The RRT$^\eta$ algorithm, outlined in Algorithm \ref{alg:RRTeta}, constructs a tree $\mathcal{T}=(\mathcal{V},\mathcal{E})$ where $\mathcal{V}$ represents the set of nodes and $\mathcal{E}=\mathcal{V}\times\mathcal{V}$ denotes the set of edges. Each node $v\in\mathcal{V}$ contains several important attributes: (i) $v.\boldsymbol{q}$ representing the system configuration; (ii) $v.\phi$ capturing the active STL specification; (iii) $v.t$ indicating when the state is reached; (iv) $v.\texttt{control\_StateTraj}$ showing control inputs and resulting system trajectory from root node to this state with $v.\texttt{control\_StateTraj}.\mathbf{u}_{0,v.t-1} = \boldsymbol{u}_{0}\boldsymbol{u}_{1}\dots\boldsymbol{u}_{v.t-1}$ and $v.\texttt{control\_StateTraj}.\mathbf{q}_{0,v.t} = \boldsymbol{q}_{1}\boldsymbol{q}_{2}\dots\boldsymbol{q}_{v.t}$; (v) $v.\texttt{parent}$ identifying the parent node; (vi) $v.\texttt{ch}$ listing children nodes; (vii) $v.[\eta]$ the AGM robustness interval.

The algorithm initializes the tree with the initial state and STL specification (Line \ref{line:RRTeta_init}). Each iteration of the main loop (Lines \ref{line:RRTeta_mainloopStart}-\ref{line:RRTeta_mainloopEnd}) expands the tree toward satisfying the specification with maximum AGM robustness.

Our algorithm leverages the established sampling procedure from Vasile et al. \cite{CristiKaraman17_STL_RRTstar}, which identifies active STL predicates at a sampled time and generates configurations within these active predicate regions. At each iteration, we invoke the $\texttt{sample}$ function (Line \ref{line:RRTeta_sample}), which returns a time-state pair $(t^r, \boldsymbol{q}^r)$ biased toward regions relevant to the STL specification $\phi$. The algorithm then determines nearby nodes $\mathcal{N}$ respecting causality constraints (Line \ref{line:RRTeta_Near}) and generates a random convex coefficient $\lambda$ (Line \ref{line:RRTeta_lambda}) balancing movement toward the sampled configuration versus along the DIAS.

For each nearby node $v^\prime \in \mathcal{N}$ (Lines \ref{line:RRTeat_bestParentStart}-\ref{line:RRTeat_bestParentEnd}), the algorithm computes the optimal control using the steering function, which applies control inputs to extend trajectories between states (detailed in Section \ref{subsec:SteeringFuns}). The resulting state $\boldsymbol{q}^s$ is evaluated (Line \ref{line:RRTeta_q_eval}), and if it improves the current best solution while being reachable via the steering function, it becomes the new best solution (Lines \ref{line:RRTeta_assingBestPStart}-\ref{line:RRTeta_assingBestPEnd}).

After adding the new node to the tree (Line \ref{line:RRTeta_updateBeforeRewirig}), the algorithm performs rewiring to improve AGM robustness of existing nodes (Lines \ref{line:RRTeta_rewiringStart}-\ref{line:RRTeta_rewiringEnd}). This is a critical step where our approach differs from traditional RRT$^\ast$: while RRT$^\ast$ rewires to minimize path cost, RRT$^\eta$ rewires to maximize AGM robustness.

\subsubsection{Steering Functions}\label{subsec:SteeringFuns}

In RRT$^\eta$, the steering function plays a critical role in extending the tree toward promising regions of the state space. We define the basic steering function $\texttt{steer}: \mathcal{Q} \times \mathcal{U} \times \mathbb{Z}_{>0} \to \mathcal{Q}$, where $\boldsymbol{q}_s \gets \texttt{steer}(\boldsymbol{q}_{t_0}, \boldsymbol{u}, \Delta t)$ applies control input $\boldsymbol{u}$ to system (\ref{eq:system}) starting from initial state $\boldsymbol{q}_{t_0}$ for duration $\Delta t$, resulting in final state $\boldsymbol{q}_s$. This function simply evolves the system dynamics under the given control input, making it the natural mechanism for generating trajectories during planning.

For exact steering between states, we also define $\texttt{steer\_exct}: \mathcal{Q} \times \mathcal{Q} \to \bigcup_{i \in \mathbb{Z}_{>0}} (\mathcal{U}^i \times \mathcal{Q}^i) \cup \{\emptyset\}$, where $\varphi \gets \texttt{steer\_exct}(\boldsymbol{q}_{\mathrm{start}}, \boldsymbol{q}_{\mathrm{final}})$ returns the control-trajectory pair $\varphi = (\mathbf{u}_{0,T-1}, \mathbf{q}_{0,T-1})$ if $\boldsymbol{q}_{\mathrm{final}}$ is reachable from $\boldsymbol{q}_{\mathrm{start}}$, and $\emptyset$ otherwise.

The guided steering using DIAS relies on the $\texttt{steer}$ function. For each nearby node $v^\prime \in \mathcal{N}$ with time difference $\Delta t^r := t^r - v^\prime.t$, we compute the optimal control input $\boldsymbol{u}^\ast$ by solving:

\begin{equation}\label{eq:computing_optimal_u}
\begin{aligned}
    \boldsymbol{u}^\ast\gets\arg\min\limits_{\boldsymbol{u}\in\mathcal{U}} J_{\chi}&(\boldsymbol{u},\texttt{steer}(v.\boldsymbol{q},\boldsymbol{u},\Delta t^r),\\&v.\boldsymbol{q},v.\phi,\boldsymbol{q}^r,\Delta t^r,\lambda)
\end{aligned}
\end{equation}where the cost function $J_\chi$ balances two objectives:
\begin{equation}\label{eq:J_chi}
\begin{aligned}
J_{\chi}(u, q_s, v^{\prime}.q, v^\prime.\phi, q^r, \Delta t^r; \lambda) := &\\
\lambda \|q_s - (v^{\prime}.q + d_\chi \cdot \Delta t^r)\|^2 + (1-&\lambda) \|q_s - q^r\|^2
\end{aligned}
\end{equation}

Here, $d_\chi := \chi_\eta(v^\prime.\boldsymbol{q}_t, v^\prime.\phi)$ represents the DIAS vector, and $\lambda \in [0, 1]$ is a weighting factor balancing movement along the DIAS direction versus toward the random sample.

Solving the optimal control problem (\ref{eq:computing_optimal_u}) exactly is generally intractable for nonlinear systems, as it requires optimizing over the control space $\mathcal{U}$ subject to nonlinear dynamics $f(\cdot, \cdot)$ from equation (\ref{eq:system}). The problem is non-convex due to the coupling between control inputs and resulting states through the $\texttt{steer}$ function, and the cost landscape may contain local minima particularly when the DIAS vector $d_\chi$ points away from the sampled configuration $q^r$.

In practice, we employ gradient-based local optimization initialized from random samples in $\mathcal{U}$. For systems with differential constraints (e.g., unicycle), we use trajectory optimization with finite-horizon discretization, while for kinematic systems (e.g., KUKA manipulator with IK), the problem reduces to selecting among feasible IK solutions based on the cost function. The computational cost per optimization is $O(N_{\text{iter}} \cdot T_f)$, where $N_{\text{iter}} \approx 10$-$50$ iterations and $T_f$ is the forward dynamics evaluation time. This local optimization approach trades global optimality for computational efficiency, which is acceptable in the RRT$^\eta$ framework since multiple steering attempts occur during tree construction, and suboptimal individual connections can be improved through subsequent rewiring (Lines \ref{line:RRTeta_rewiringStart}--\ref{line:RRTeta_rewiringEnd}).

\subsubsection{AGM-Based Rewiring}\label{subsubsec:rewiring}

The algorithm uses our $\texttt{update}_{\eta}$ procedure, detailed in Algorithm \ref{alg:updateAGM}, which extends the update procedure from \cite{CristiKaraman17_STL_RRTstar} to work with AGM robustness intervals. As shown in Algorithm \ref{alg:updateAGM}, we process each observation in the trajectory incrementally using our $\texttt{IRTM}$ function (Algorithm \ref{alg:EfficientIncrem_IRTM}), which efficiently computes and updates AGM robustness intervals as new observations become available.

The key innovation in our update procedure is the use of AGM robustness intervals as the criterion for adding nodes and rewiring the tree. Nodes are only added when the AGM robustness interval has a positive upper bound, indicating specification satisfaction remains feasible. For rewiring, our algorithm selects connections that improve the lower bound $\underline{\eta}$ of the AGM robustness interval while maintaining formula consistency.

The lower bound optimization strategy provides conservative guarantees during tree construction. Since $\underline{\eta}$ represents the worst-case robustness over all possible trajectory completions (Lemma \ref{lemma:rosiSoundness}), maximizing $\underline{\eta}$ ensures that even the least favorable completion maintains high robustness. This conservative approach prevents the algorithm from committing to partial trajectories that appear promising based on optimistic upper bounds $\overline{\eta}$ but may lead to low-robustness or infeasible completions. As trajectories extend toward the formula horizon $||\phi||$, the interval converges to the exact AGM robustness value (Corollary \ref{corollary:AGM_convgnc}), eliminating conservatism for complete paths while maintaining safety during exploration.

By basing all tree construction decisions on AGM robustness rather than traditional robustness, our approach maintains the algorithmic structure of STL-RRT$^\ast$ while incorporating the benefits of considering all subformulae and time points, not just the critical ones. This results in trajectories that not only satisfy specifications but do so with maximum robustness across all specification aspects.





\setlength{\textfloatsep}{15pt}
\begin{algorithm}[t]\scriptsize
    \bf{Input:} $\boldsymbol{q}_{init}$ -- Initial configuration \\
    \bf{Input:} $\phi$ -- STL formula in positive normal form\label{line:RRTeta_init}\\
    \bf{Output:} \textbf{u} -- a satisfying control policy w.r.t. $\phi$ with maximum AGM robustness\\
    $\mathcal{T} = (V := \emptyset, E = \emptyset)$\\
    $V \gets (v.\boldsymbol{q}_{init}\gets\boldsymbol{q}_{init},\;v. \phi\gets\phi,\;v.t\gets 1,\;v.\texttt{control\_StateTraj}.\mathbf{q}\gets\emptyset,\;v.\texttt{parent}\gets\emptyset,v.\texttt{ch}\gets\emptyset)$\\
    \For{$k = 1 : N^{max}$}
    {\label{line:RRTeta_mainloopStart}
        $t^r, \boldsymbol{q}^r \gets \texttt{sample}(\mathcal{Q}, \mathcal{T}, \phi)$\label{line:RRTeta_sample}\\
        $\mathcal{N} \gets \texttt{near}(\mathcal{T}, \boldsymbol{q}^r, t^r)$\label{line:RRTeta_Near}\\
        $\lambda \gets \texttt{Unif}([0,1])$\label{line:RRTeta_lambda}\\
        $v.\texttt{parent} \gets \emptyset, J^\ast \gets \infty, \boldsymbol{q}^\ast \gets \emptyset$\\
        \ForEach{$v^\prime \in \mathcal{N}$}
        {\label{line:RRTeat_bestParentStart}
            $\Delta t^r = t^r - v^\prime.t$\\
            // Compute DIAS vector using AGM robustness intervals
            $d_{\chi} \gets \chi_\eta(v^\prime.\texttt{control\_stateTraj}.q_{v^\prime.t}, v^\prime.\phi, v^\prime.t)$\\
            // Compute optimal control using DIAS and random sample
            $\boldsymbol{u}^\ast\gets \arg\min\limits_{\boldsymbol{u}\in\mathcal{U}} J_{\chi}(\boldsymbol{u},\texttt{steer}(v.\boldsymbol{q},\boldsymbol{u},\Delta t^r),v.\boldsymbol{q},v.\phi,\boldsymbol{q}^r,\Delta t^r,\lambda)$\label{line:RRTeta_uOptimal}\\
            $\boldsymbol{q}^s \gets \texttt{steer}(v^\prime.\boldsymbol{q},\boldsymbol{u}^\ast,\Delta t^r)$\\
            $J^s \gets J_{\mathbf{q}}(\boldsymbol{u}^\ast, \boldsymbol{q}^s, v^\prime.\boldsymbol{q}, v^\prime.\phi, \boldsymbol{q}^r, d_\eta, \lambda)$\label{line:RRTeta_q_eval}\\
            \If{$J^s < J^\ast \wedge \texttt{steer\_exct}(v^\prime.\boldsymbol{q}, \boldsymbol{q}^s)$}
            {\label{line:RRTeta_assingBestPStart}
                $J^\ast \gets J^s, v.\texttt{parent} \gets v^\prime, \boldsymbol{q}^\ast \gets \boldsymbol{q}^s$
            }\label{line:RRTeta_assingBestPEnd}
        }\label{line:RRTeat_bestParentEnd}
        $v_{\mathrm{temp}}.\boldsymbol{q}\gets\boldsymbol{q}^\ast,\;v_{\mathrm{temp}}.\phi\gets\emptyset,\;v_{\mathrm{temp}}.t\gets\emptyset,\;v_{\mathrm{temp}}.\texttt{control\_StateTraj}\gets\emptyset,\;v_{\mathrm{temp}}.\texttt{parent}\gets\emptyset,\;v_{\mathrm{temp}}.\texttt{ch}\gets\emptyset,\;v_{\mathrm{temp}}.[\eta]\gets[-1,1]$\\
        // Update with AGM robustness interval calculation
        $\texttt{update}_{\eta}(v.\texttt{parent}, v_{\mathrm{temp}})$\label{line:RRTeta_updateBeforeRewirig}\\
        \For{$v^{\prime\prime} \in Near(\mathcal{T}, \boldsymbol{q}^\ast, t^r)$}
        {\label{line:RRTeta_rewiringStart}
            \If{$\texttt{steer\_exct}(\boldsymbol{q}^\ast, \boldsymbol{q}^{\prime\prime})$}
            {
                // Rewiring based on AGM robustness intervals
                $\texttt{update}_{\eta}(v_{\texttt{temp}}, v^{\prime\prime})$\label{line:RRTeta_updateWhenRewiring}
            }
        }\label{line:RRTeta_rewiringEnd}
        \If{$\texttt{existsSolutionAGM}()$}
        {\label{line:RRTeta_findingSolutionStart}
            $v_{\mathrm{best}} = \texttt{bestAGM}(\mathcal{V})$\\
            \Return $v_{\mathrm{best}}.\texttt{control\_StateTraj}$
        }
        \Else{\Return $\emptyset$}\label{line:RRTeta_findingSolutionEnd}
    }\label{line:RRTeta_mainloopEnd}
    \caption{RRT$^{\eta}$ Algorithm}
    \label{alg:RRTeta}
\end{algorithm}
\setlength{\textfloatsep}{14pt}

\begin{algorithm}[t]\scriptsize
    \caption{$\texttt{update}_{\eta}(v_1, v_2)$}
    \label{alg:updateAGM}
    // Initialize AGM robustness interval with parent's interval\\
    $[\eta]_{\mathrm{curr}} \gets v_1.[\eta]$\\
    $t_{\mathrm{curr}} \gets v_1.t$\\
    // Get trajectory points from steering function\\
    $(\boldsymbol{u}_{1,2}\mathbf{q}_{1,2}) \gets \texttt{steer\_exct}(v_1.\boldsymbol{q}, v_2.\boldsymbol{q})$\\
    // Incrementally compute AGM robustness interval for each observation\\
    \For{$i = 1$ to $n$}
    {
        $[\eta]_{\mathrm{curr}} \gets$ \texttt{IRTM}$(v_1.\phi, [\eta]_{\mathrm{curr}}, \mathbf{q}_{1,2}(i), t_{\mathrm{curr}})$\\
        $t_{\mathrm{curr}} \gets t_{\mathrm{curr}} + 1$\\
    }
    $[\underline{\eta}^{\prime}_2, \overline{\eta}^{\prime}_2] \gets [\eta]_{\mathrm{curr}}$\\
    \If{$\phi_2 = \emptyset$}
    {
        \If{$\overline{\eta}^{\prime}_2 \geq 0$}
        {
            $v_2.[\eta] \gets [\underline{\eta}^{\prime}_2, \overline{\eta}^{\prime}_2]$\\
            $v_2.\phi \gets \texttt{simplify}(v_1.\phi)$\\
            $\mathcal{V} \gets \mathcal{V} \cup \{v_2\}$, $\mathcal{E} \gets \mathcal{E} \cup \{(v_1, v_2)\}$
        }
    }
    \ElseIf{$\overline{\eta}^{\prime}_2 \geq 0 \wedge \underline{\eta}^{\prime}_2 \geq \min(v_2.[\eta])\wedge \phi_1 \Rightarrow \phi_2$}
    {
        $v_2.[\eta] \gets [\underline{\eta}^{\prime}_2, \overline{\eta}^{\prime}_2]$\\
        $\mathcal{E} \gets (\mathcal{E} \setminus \{(v_2.\texttt{parent}, v_2)\}) \cup \{(v_1, v_2)\}$\\
        $V_{upd} = ch(v_2)$ // children of $v_2$\\
        \While{$V_{upd} \neq \emptyset$}
        {
            $v \gets V_{upd}.pop()$, $v^{\prime} \gets v.\texttt{parent}$\\
            // Recompute intervals for affected branches recursively\\
            $[\eta]_{\mathrm{curr}} \gets v^{\prime}.[\eta]$\\
            $t_{\mathrm{curr}} \gets v^{\prime}.t$\\
            ($\boldsymbol{u}_{v^{\prime},v},\mathbf{q}_{v^{\prime}.t,v.t}) \gets \texttt{steer\_exct}(v^{\prime}.\boldsymbol{q},v.\boldsymbol{q})$\\
            \For{$i = 1$ to $|\mathbf{q}_{v^{\prime}.t,v.t}|$}
            {
                $[\eta]_{\mathrm{curr}} \gets$ \texttt{IRTM}$(v^{\prime}.\phi, [\eta]_{\mathrm{curr}}, \mathbf{q}_{v^\prime.t,v.t}(i), t_{\mathrm{curr}})$\\
                $t_{\mathrm{curr}} \gets t_{\mathrm{curr}} + \Delta t$\\
            }
            $v.[\eta] \gets [\eta]_{\mathrm{curr}}$
        }
    }
\end{algorithm}

\subsection{Optimization Landscape Advantages}\label{subsec:landscape}
While gradient computation for predicates remains similar between traditional and AGM approaches, a significant advantage emerges during tree rewiring. The cost function used for rewiring in RRT$^\eta$ is based on AGM robustness, which creates a notably smoother optimization landscape compared to traditional robustness. This smoothness arises not from the gradient calculation itself, but from how AGM robustness aggregates satisfaction values across all subformulae and time points.

Traditional robustness, with its $\min/\max$ operators, creates sharp transitions in the optimization landscape whenever the critical subformula changes. In contrast, our AGM approach produces a more continuous landscape through arithmetic and geometric means, making the rewiring process more effective at finding high-quality trajectories. This smoother landscape helps RRT$^\eta$ avoid getting trapped in local optima that might occur with the sharp decision boundaries of traditional robustness.

\subsection{Optimality and Completeness}\label{subsec:ARMORRRTeta_completenss}

The RRT$^\eta$ algorithm combines the scalability of RRT$^\ast$ with the expressiveness of STL specifications using AGM robustness. A key question is whether this integration preserves the core theoretical guarantees of sampling-based planning. We now establish that the algorithm maintains fundamental properties of probabilistic completeness and asymptotic optimality despite the more complex robustness measure.

\begin{theorem}[Probabilistic Completeness]\label{thm:rrteta_completeness}
Consider the RRT$^\eta$ algorithm (Algorithm \ref{alg:RRTeta}) applied to a dynamical system with Lipschitz continuous dynamics (\ref{eq:system}), an STL specification $\phi$, and initial state $\boldsymbol{q}_{\mathrm{init}}$. Let $\mathcal{G}$ denote the set of feasible control-trajectory pairs with positive AGM robustness.

If $\mathcal{G} \neq \emptyset$, then
\begin{equation}
\lim_{n \to \infty} \mathbb{P}(\text{RRT}^\eta \text{ finds } \varphi \in \mathcal{G}) = 1
\end{equation}
where $n$ is the number of iterations.
\end{theorem}

\begin{proof}[Proof Sketch]
The proof relies on two key properties:

\textbf{Stochastic Exploration.} The sampling procedure (Line \ref{line:RRTeta_sample}) combined with randomized composition mechanisms ensures non-zero probability of exploring all regions of the state space. For FPL-based composition, the randomization term $\alpha_i$ in Equation (\ref{eq:randomization}) maintains exploration. For stochastic composition, the Bernoulli choice in Equation (\ref{eq:choose}) ensures no region is permanently excluded.

\textbf{Sound Robustness Bounds.} By Lemma \ref{lemma:rosiSoundness}, the AGM robustness intervals computed by $\texttt{IRTM}$ (Algorithm \ref{alg:EfficientIncrem_IRTM}) are sound: for any completion $\mathbf{s} \in \mathfrak{C}(\mathbf{s}_{t_0,t'})$, we have $\eta(\mathbf{s}, \phi) \in [\eta]_{\mathbf{s}_{t'},\phi}$. This ensures that the algorithm correctly identifies when a partial trajectory can potentially reach positive AGM robustness, preventing premature pruning of feasible paths.

Given sufficient sampling time, the stochastic exploration guarantees that feasible regions are eventually sampled, and sound robustness bounds ensure these samples are correctly identified and added to the tree.
\end{proof}

\begin{theorem}[Asymptotic Optimality]\label{thm:rrteta_optimality}
Under the conditions of Theorem \ref{thm:rrteta_completeness}, let $\varphi^\ast = \argmax_{\varphi \in \mathcal{G}} \eta(\mathbf{q}_{0,T}, \phi)$ denote the optimal solution to Problem \ref{pr:opt_control_problem}. Let $\varphi^n = (\mathbf{u}_{0,T}^n, \mathbf{q}_{0,T}^n)$ denote the control-trajectory pair returned by RRT$^\eta$ after $n$ iterations.

Then, for any $\epsilon > 0$,
\begin{equation}
\lim_{n \to \infty} \mathbb{P}(\eta(\mathbf{q}_{0,T}^n, \phi) \geq \eta(\mathbf{q}_{0,T}^\ast, \phi) - \epsilon) = 1
\end{equation}
\end{theorem}

\begin{proof}[Proof Sketch]
The proof builds on completeness (Theorem \ref{thm:rrteta_completeness}) and establishes convergence to optimality through:

\textbf{Monotonic Refinement.} By Theorem \ref{thm:rosiChainInclusion}, as partial trajectories extend, AGM robustness intervals satisfy $[\eta]_{\mathbf{s}_{t'},\phi} \subseteq [\eta]_{\mathbf{s}_{t},\phi}$ for $t' > t$. This monotonicity ensures that longer trajectories provide increasingly precise robustness estimates, allowing the algorithm to systematically identify and pursue high-robustness paths through the rewiring mechanism (Lines \ref{line:RRTeta_rewiringStart}--\ref{line:RRTeta_rewiringEnd}).

\textbf{Convergence to Exact Values.} By Corollary \ref{corollary:AGM_convgnc}, when a trajectory reaches $t' \geq ||\phi||$, the interval converges to exact AGM robustness: $[\eta]_{\mathbf{s}_{t'},\phi} = \{\eta(\mathbf{s}, \phi)\}$. This enables accurate comparison and selection of optimal solutions.

\textbf{Rewiring with AGM Cost.} The algorithm's rewiring procedure (Algorithm \ref{alg:updateAGM}) uses AGM robustness as the cost metric, selecting parent connections that maximize $\underline{\eta}$ (the lower bound of the robustness interval) while maintaining formula consistency. Combined with the continuous exploration from stochastic mechanisms, this ensures that as $n \to \infty$, the tree progressively improves toward the optimal AGM robustness value.

The formal proof parallels the RRT$^\ast$ optimality proof \cite{KaramanRRTstarIJRR}, substituting path length cost with AGM robustness maximization and leveraging the properties established in Lemma \ref{lemma:rosiSoundness}, Theorem \ref{thm:rosiChainInclusion}, and Corollary \ref{corollary:AGM_convgnc}.
\end{proof}


  


\section{Case Studies}\label{sec:example}

\subsection{Unicycle Robot}

We evaluate our AGM-based planning approach on a unicycle-drive robot operating 
in a planar environment with sequential visitation requirements and continuous 
avoidance constraints. This benchmark demonstrates the effectiveness of AGM 
robustness on systems with nonholonomic constraints and strict temporal ordering.

\textbf{System Dynamics.} 
Consider a unicycle robot with configuration $\boldsymbol{q} = [x, y, \theta]^\top 
\in \mathbb{R}^2 \times [-\pi, \pi]$, where $(x,y) \in \mathbb{R}^2$ represents 
the robot position and $\theta \in [-\pi, \pi]$ is the heading angle. The robot 
evolves according to discrete-time dynamics with sampling time $\Delta t$:
\begin{equation}\label{eq:unicycle_dynamics}
\begin{aligned}
x_{t+1} &= x_t + v_t \cos(\theta_t) \Delta t \\
y_{t+1} &= y_t + v_t \sin(\theta_t) \Delta t \\
\theta_{t+1} &= \theta_t + \omega_t \Delta t \\
v_{t+1} &= u_1, \quad \omega_{t+1} = u_2
\end{aligned}
\end{equation}
where $v \in [-0.3, 0.3]$ m/s is the translational velocity and $\omega \in [-1, 1]$ 
rad/s is the angular velocity. The augmented state space: $\mathbf{x} = 
(x, y, \theta, v, \omega)^\top \in \mathbb{R}^5$.

The planar workspace $\mathcal{W} = [0, 4] \times [0, 4]$ m$^2$ contains three 
rectangular regions: textbf{Region 1} (initial target): $x \in [2.0, 3.0]$, $y \in [1.0, 2.0]$; \textbf{Region 2} (final target): $x \in [0.5, 1.5]$, $y \in [2.5, 3.0]$; and \textbf{Obstacle region} (forbidden): $x \in [0.5, 1.5]$, $y \in [1.0, 2.0]$.

The target region predicates are formulated as conjunctions of linear inequalities: $\mu_{\text{Region1}} := (2.0 \leq x \leq 3.0) \wedge (1.0 \leq y \leq 2.0$, and $\mu_{\text{Region2}} := (0.5 \leq x \leq 1.5) \wedge (2.5 \leq y \leq 3.0)$. The obstacle avoidance constraint is formulated as: $\mu_{\text{avoid}} := (x < 0.5) \vee (x > 1.5) \vee (y < 1.0) \vee (y > 2.0)$.

The task requires sequential region visitation with temporal ordering and 
continuous obstacle avoidance:
\begin{equation}\label{eq:unicycle_specification}
\begin{aligned}
\phi_{\text{unicycle}} := \mathbf{F}_{[0, 15]}(&\mu_{\text{Region1}})
\wedge \mathbf{F}_{[15, 40]}(\mu_{\text{Region2}}) \\
&\wedge \mathbf{G}_{[0, 20]}(\mu_{\text{avoid}})
\end{aligned}
\end{equation}

The obstacle region is positioned directly between the two target regions, requiring 
the robot to navigate around it. The nonholonomic constraints further complicate 
planning—the robot cannot move sideways and must carefully coordinate heading and 
velocity commands.

AGM Robustness Enables Feasible Solutions. Figure~\ref{fig:unicycle_results} reveals dramatic performance differences between traditional and AGM-based approaches. Traditional STL-RRT* using standard robustness semantics demonstrates catastrophic failure: even after normalization from $[-4.4] \rightarrow [-1.1]$ to account for different robustness scales, the 
lower bound remains persistently negative ($\underline{\eta} \approx -0.35$), indicating the planner cannot discover any trajectory satisfying the specification. The upper bound reaches only $\overline{\eta} \approx 0.05$, and the gap plateaus at $\approx 0.4$, confirming the planner is trapped exploring infeasible state space.

The fundamental limitation of traditional robustness in this scenario stems from its min-max semantics: when evaluating paths through the temporal sequence, traditional robustness takes the minimum over all time steps, causing any momentary low-robustness configuration (e.g., when navigating near the obstacle 
boundary) to dominate the entire trajectory evaluation. This pessimistic semantics prevents the planner from recognizing that brief proximity to constraint boundaries can be acceptable if compensated by high robustness elsewhere.

In stark contrast, both AGM-based heuristics successfully find high-quality solutions. Choose-blend (red) achieves $\underline{\eta} \approx 0.93$ and $\overline{\eta} \approx 0.90$ with gap $\approx 0.1$, while FPL (green) reaches $\underline{\eta} \approx 0.95$ and $\overline{\eta} \approx 0.98$ with gap $< 0.05$. The additive structure of AGM robustness enables both methods to balance constraint satisfaction across the trajectory: states with moderate robustness near the obstacle can be accepted if they enable high-robustness states in the target regions, and this trade-off is quantified explicitly through AGM's aggregation mechanism.

The convergence patterns (panels a-b) demonstrate this fundamental difference: AGM methods show rapid, monotonic improvement, indicating systematic discovery of increasingly robust trajectories. Traditional robustness shows no such improvement pattern, remaining stuck in low-robustness regions throughout planning.

While both AGM methods vastly outperform traditional robustness, FPL demonstrates superior efficiency through gradient-based objective balancing. When evaluating competing subformulae during tree expansion—such as whether to prioritize reaching Region A or Region B during $t \in [2, 7]$—FPL computes fulfillment values $f_i$ from the current AGM robustness intervals (Equation \ref{eq:fulfillment_mapping}). These fulfillment values indicate how well each subformula is currently satisfied. FPL then uses power mean derivatives (Equation \ref{eq:fpl_weights}) to compute weights $w_i$ that naturally prioritize less-fulfilled objectives, guiding exploration toward states that balance all specification requirements rather than optimizing only the most critical constraint.

This gradient-based balancing proves particularly effective for sequential tasks. When expanding the tree near Region A at $t = 5$, the fulfillment-based weights guide sampling toward configurations that not only satisfy the immediate requirement (reaching Region A) but also maintain flexibility for satisfying subsequent requirements (later reaching Region D or E while avoiding obstacles). Choose-blend lacks this principled balancing mechanism and must discover good paths through repeated stochastic sampling and rewiring. The gap metric (panel c) quantifies this difference: FPL reaches gap $< 0.05$ within 400 iterations ($\sim 17$s), while choose-blend achieves gap $\approx 0.1$ at 800 iterations ($\sim 35$s), demonstrating FPL's $2\times$ computational advantage among AGM methods.




Choose-blend lacks this predictive capability and must discover good sequencing through repeated sampling. The gap metric (panel c) quantifies the efficiency difference: FPL reaches gap $< 0.05$ within 400 iterations ($\sim 17$s), while choose-blend achieves gap $\approx 0.1$ at 800 iterations ($\sim 35$s), demonstrating FPL's $2\times$ computational advantage among AGM methods.
The tree visualization (panel d) shows the final solution path in the $(x_1, x_2)$ plane. Both AGM methods discover trajectories that successfully reach Region 1 (gray rectangle, right-center), navigate around the obstacle region (beige rectangle, left-center) via an upper arc, and terminate in Region 2 (orange rectangle, upper-left). The smooth trajectory with color-coded temporal progression demonstrates feasible motion respecting nonholonomic constraints—motion that traditional robustness-based planning fails to discover.

This case study demonstrates three key advantages of AGM robustness. First, AGM methods discover satisfying solutions where traditional robustness fails entirely—the traditional approach remained trapped in infeasible regions with negative lower bounds throughout planning while AGM methods achieved high-quality solutions with $\underline{\eta} \approx 0.93$-$0.95$. Second, both AGM methods converge orders of magnitude faster than traditional approaches, reaching positive robustness within hundreds of iterations compared to the traditional method's persistent failure. Third, among AGM methods, FPL's gradient-based objective balancing provides additional computational advantages, achieving convergence in half the iterations required by choose-blend while maintaining comparable or superior solution quality.

\begin{figure}
    \centering
    \includegraphics[width=.6\columnwidth]{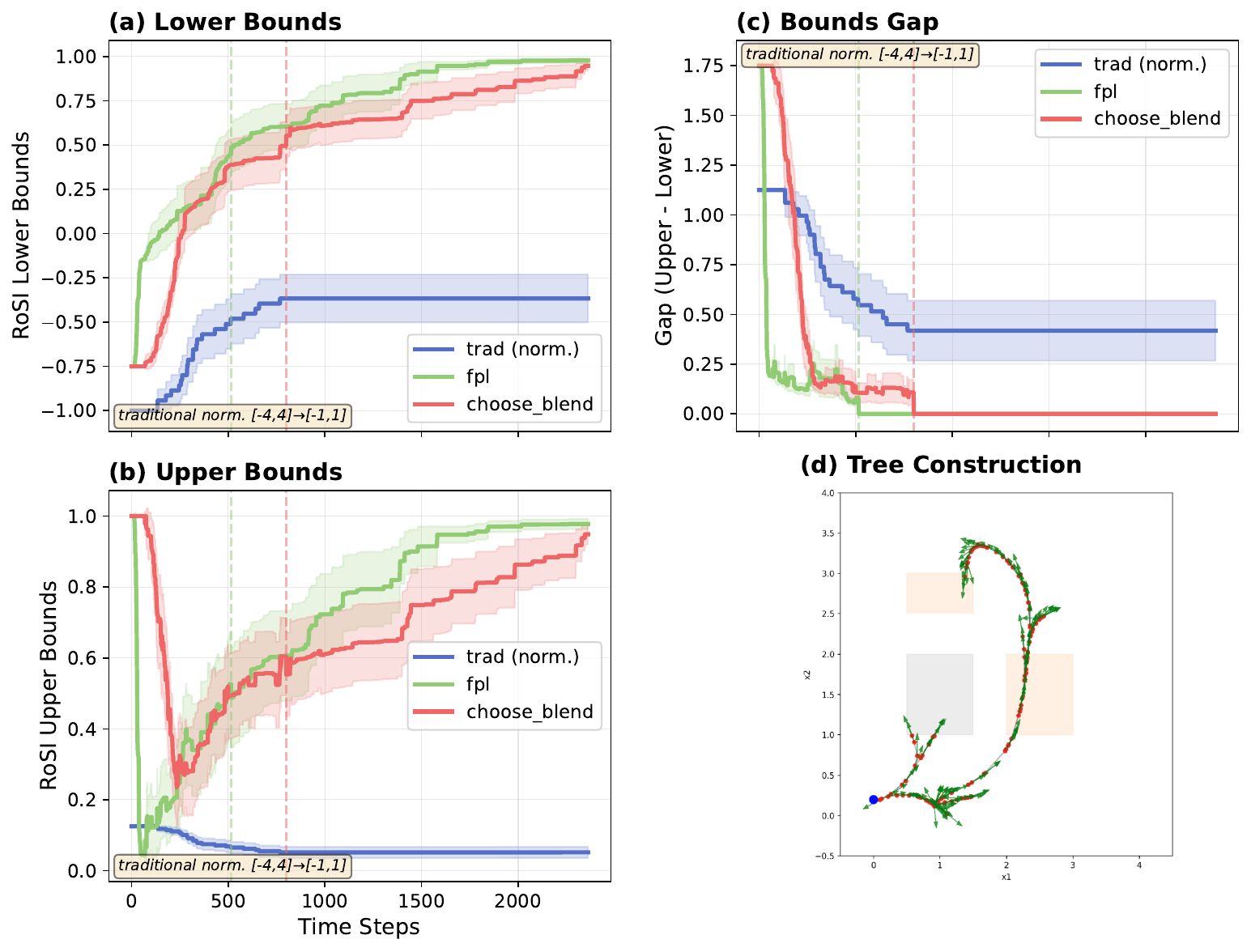}
    \caption{Performance comparison on unicycle robot sequential reach-avoid task, 
    demonstrating critical advantages of AGM robustness. (a-c) Traditional 
    robustness (blue, normalized from $[-4.4] \rightarrow [-1.1]$) completely 
    fails with negative lower bound ($\underline{\eta} \approx -0.35$), near-zero upper 
    bound ($\overline{\eta} \approx 0.05$), and stagnant gap ($\approx 0.4$), unable to 
    discover any satisfying trajectory. Both AGM-based methods succeed: 
    choose-blend (red) achieves $\underline{\eta} \approx 0.93$, $\overline{\eta} \approx 0.90$, 
    gap $\approx 0.1$; FPL (green) reaches $\underline{\eta} \approx 0.95$, $\overline{\eta} \approx 
    0.98$, gap $< 0.05$. The additive structure of AGM robustness enables both 
    methods to balance constraint satisfaction across temporal phases, discovering 
    feasible trajectories that traditional min-max semantics cannot recognize. 
    Among AGM methods, FPL's forward prediction provides $2\times$ computational 
    advantage, reaching convergence at 400 iterations vs. 800 for choose-blend. 
    (d) Tree construction shows solution path (color indicates temporal 
    progression: blue $\rightarrow$ yellow $\rightarrow$ red). The specification 
    requires visiting Region 1 (gray, $[2.0,3.0] \times [1.0,2.0]$m) during 
    $t \in [0,15]$, then Region 2 (orange, $[0.5,1.5] \times [2.5,3.0]$m) during 
    $t \in [15,40]$, while avoiding obstacle region (beige, $[0.5,1.5] \times 
    [1.0,2.0]$m) for $t \in [0,20]$. The successful upper-arc trajectory demonstrates 
    AGM's ability to discover feasible solutions for complex temporal-spatial 
    constraints where traditional approaches fail entirely.}
    \label{fig:unicycle_results}
\end{figure}

\subsection{7DOF Manipulator}

We evaluate our approach on a 7-DOF KUKA iiwa manipulator operating in a 
constrained workspace with multiple target regions and an obstacle. This scenario presents a \textit{cascading choice problem} where the planner must 
make sequential decisions about which regions to visit while satisfying temporal constraints and continuous safety requirements.

\textbf{Problem Setup.} 
The workspace $\mathcal{W} \subset \mathbb{R}^3$ contains four target regions defined as 3D bounding boxes and one spherical obstacle as follows. \textbf{Region A} (yellow): $x \in [0.40, 0.55]$m, $y \in [0.15, 0.30]$m, $z \in [0.50, 0.65]$m; \textbf{Region B} (cyan): $x \in [-0.10, 0.10]$m, $y \in [0.25, 0.40]$m, $z \in [0.55, 0.70]$m; \textbf{Region D} (blue): $x \in [0.45, 0.60]$m, $y \in [-0.10, 0.10]$m, 
    $z \in [0.65, 0.80]$m; \textbf{Region E} (magenta): $x \in [-0.15, 0.15]$m, $y \in [0.35, 0.50]$m, 
    $z \in [0.75, 0.90]$m; and \textbf{Obstacle} (black sphere): center at $(0.20, 0.30, 0.60)$m with radius $r_{\text{obs}} = 0.12$m. 

Each target region is formalized as a workspace predicate:
\begin{equation}
\begin{aligned}
    \mu_i := (x \in [x_i^{\min}, &x_i^{\max}]) \wedge (y \in [y_i^{\min}, y_i^{\max}])\\& 
\wedge (z \in [z_i^{\min}, z_i^{\max}])
\end{aligned}
\end{equation}
for $i \in \{A, B, D, E\}$. The obstacle avoidance constraint ensures safe clearance:
\begin{equation}
\begin{aligned}
    \mu_{\text{obs-free}} := &\sqrt{(x-0.20)^2 + (y-0.30)^2 + (z-0.60)^2} \\ & \geq  r_{\text{obs}} + d_{\text{safe}}
\end{aligned}
\end{equation}
where $d_{\text{safe}} = 0.03$m provides a safety margin, requiring the end-effector 
to maintain at least $0.15$m distance from the obstacle center.

The task requires sequential region visitation with temporal constraints:
\begin{equation}\label{eq:kuka_specification}
\begin{aligned}
\phi_{\text{KUKA}} := &\mathbf{F}_{[2, 7]}(\mu_A \vee \mu_B) 
\wedge \mathbf{F}_{[8,15]}(\mu_D \vee \mu_E) \\
&\wedge \mathbf{G}_{[0, 15]}(\mu_{\text{obs-free}}) 
\wedge \mathbf{G}_{[0, 15]}(\mu_{\text{joint-limits}})
\end{aligned}
\end{equation}
where $\mu_{\text{obs-free}}$ ensures the end-effector maintains safe distance 
$d > d_{\text{safe}} = 0.15$m from the obstacle, and $\mu_{\text{joint-limits}} := 
\bigwedge_{i=1}^{7} (q_i \in [q_i^{\min}, q_i^{\max}])$ enforces joint limits.

This specification creates a complex decision-making scenario with four possible solution paths: A-then-D, A-then-E, B-then-D, or B-then-E. However, not all paths are feasible due to kinematic constraints and the obstacle placement. The planner must choose between visiting Region A ($0.15$m wide box on the right) or Region B ($0.20$m wide box in the center) during the first time window $t \in [2,7]$s, and this choice constrains subsequent options for visiting Region D ($0.15$m wide box on the right) or Region E ($0.30$m wide box in the center) during $t \in [8,15]$s. The obstacle positioned at $(0.20, 0.30, 0.60)$m with $0.12$m radius creates a critical constraint—certain region pairs require trajectories that pass dangerously close to or through the obstacle sphere.

The critical distinction between traditional robustness and our AGM robustness formulation becomes evident in this scenario. Traditional STL-RRT* using standard robustness semantics must explore all four region pairs with roughly equal priority, as the traditional min-max semantics provide limited guidance about which choices lead to higher overall robustness. The traditional method (blue in Figure~\ref{fig:kuka_results}) demonstrates this undirected exploration: after normalization to account for its different robustness scale ($[-4.4] \rightarrow [-1.1]$), the lower bound plateaus at $\underline{\eta} \approx -0.35$ and upper bound at $\overline{\eta} \approx 0.08$, with gap $\approx 1.2$ indicating continued exploration of low-quality solutions.

In contrast, both AGM-based heuristics—choose-blend (red) and FPL (green)—leverage the additive structure of AGM robustness to systematically resolve the choice dilemma. Rather than treating each disjunction as an isolated decision, AGM robustness accumulates contributions from both branches, allowing the planner to recognize that certain region pairs maintain consistently higher robustness throughout the trajectory. Both AGM methods demonstrate structured convergence: choose-blend achieves $\underline{\eta} \approx 0.60$ and $\overline{\eta} \approx 0.80$ with gap $\approx 0.8$, while FPL reaches $\underline{\eta} \approx 0.95$ and $\overline{\eta} \approx 0.98$ with gap $< 0.1$, both substantially outperforming the traditional baseline.

The RoSI bound evolution (panels a-b) reveals the structured nature of AGM-guided exploration: both choose-blend and FPL show monotonic convergence patterns, indicating they systematically improve solution quality by exploiting the additive robustness structure. The traditional method's erratic, stagnant bounds demonstrate its inability to leverage choice structure effectively.

 While both AGM methods significantly outperform the traditional baseline, FPL demonstrates superior performance over choose-blend through its forward predictive capability. When sampling states in Region A during $t \in [2,7]$, FPL evaluates not only the immediate AGM robustness but also predicts the future AGM robustness for reaching Regions D or E during $t \in [8,15]$. This forward-looking strategy enables FPL to prioritize the A-then-E path, which maintains maximum clearance from the obstacle throughout both temporal phases.

Choose-blend, while benefiting from AGM's additive structure, lacks this predictive horizon and must discover good paths through repeated sampling and rewiring. The gap metric (panel c) quantifies this difference: FPL reaches gap $< 0.1$ within 1000 iterations ($\sim 70$s), whereas choose-blend requires $\sim 2500$ iterations ($\sim 100$s) to achieve gap $\approx 0.2$, demonstrating FPL's $1.5\times$ computational advantage among AGM methods.

The augmented state representation $\mathbf{x} = (q_1, \ldots, q_7, x, y, z, \psi_r, \psi_p, \psi_y)^\top \in \mathbb{R}^{13}$ provides crucial computational advantages for this high-dimensional problem. This representation combines joint angles $q \in \mathcal{Q} \subset \mathbb{R}^7$ with end-effector workspace coordinates computed via the forward kinematics (FK) mapping $\mathcal{F}: \mathcal{Q} \to \mathbb{R}^3 \times \text{SO}(3)$, where $\mathcal{F}(q) = (x, y, z, \psi_r, \psi_p, \psi_y)^\top$ represents the end-effector pose.

Rather than sampling blindly in the 7D joint space $\mathcal{Q}$, we employ \textit{task-space sampling with inverse kinematics (IK)}, where IK solves for joint configurations $q = \mathcal{F}^{-1}(w)$ that achieve desired workspace poses $w$. When specification predicates require end-effector positions in workspace regions, we:
\begin{enumerate}
    \item Sample candidate workspace poses $w = (x,y,z,\mathcal{O})$ directly within target region bounds (e.g., uniformly in $[0.40, 0.55] \times [0.15, 0.30] \times [0.50, 0.65]$ for Region A)
    \item Solve IK to compute joint configurations: $q = \mathcal{F}^{-1}(w)$
    \item Construct augmented state $\mathbf{x} = [q, w]$ caching workspace coordinates
\end{enumerate}

This approach offers two significant benefits: (1) sampling is biased toward task-relevant workspace regions rather than the larger joint space, dramatically improving the probability of generating states that satisfy the specification, and (2) forward kinematics computations during robustness evaluation are eliminated since workspace coordinates are pre-computed and cached in $\mathbf{x}$. For each state $\mathbf{x}$, predicate evaluation $\eta(\mathbf{x}_t, \mu_{\text{region}})$ reduces to comparing cached $(x,y,z)$ values against region bounds—simple arithmetic operations rather than expensive trigonometric computations required by $\mathcal{F}(q)$. See Appendix \ref{sec:stl_manipulator} for detailed formalization of the augmented state representation and adaptive sampling strategy (Algorithm \ref{alg:ik_sampling}).

Let $N$ denote the number of tree nodes and $m$ the specification depth (number of temporal operators). Traditional STL-RRT* with naive joint-space sampling performs $O(Nm)$ forward kinematics computations during rewiring, as each potential parent evaluation requires computing $\mathcal{F}(q)$ for robustness checking. Our IK-based sampling with caching reduces this to $O(N)$ forward kinematics computations (once per state creation) plus $O(Nm)$ cached lookups. For the KUKA iiwa, forward kinematics computation via Denavit-Hartenberg transformations takes approximately $50\mu$s while cached lookups require $< 1\mu$s, yielding a $50\times$ speedup in predicate evaluation. Combined with FPL's faster convergence (requiring $\sim 60\%$ fewer iterations to achieve gap $< 0.1$ compared to traditional methods), the overall computational improvement is substantial—planning times are reduced from $\sim 180$s to $\sim 70$s for convergence to high-quality solutions.

Figure~\ref{fig:kuka_results}(d-f) visualizes the final reference trajectory from multiple viewpoints using ghost trail epresentation. The robot successfully visits Region A (yellow, $0.15$m $\times$ $0.15$m $\times$ $0.15$m box on the right) during the first temporal window and terminates at Region E (magenta, $0.30$m $\times$ $0.15$m $\times$ $0.15$m box in the center-top) in the second window, maintaining safe clearance from the obstacle sphere throughout execution. The smooth progression of overlaid poses demonstrates feasible motion respecting joint velocity constraints, with the final fully-opaque pose confirming specification satisfaction at $t=15$s.

\begin{figure}
    \centering
    \includegraphics[width=0.7\columnwidth]{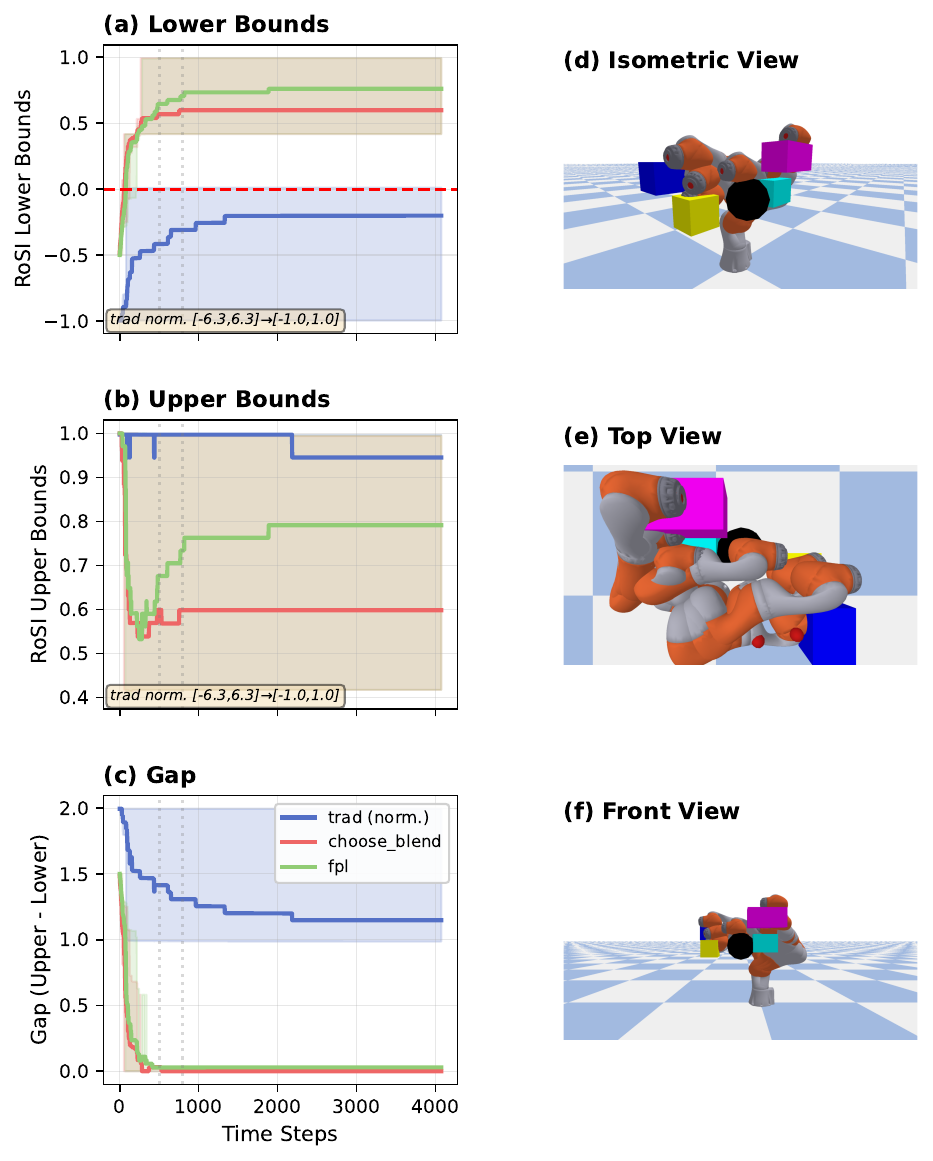}
    \caption{Performance comparison of STL-RRT* heuristics on the KUKA iiwa 
    cascading choice problem. (a-b) Lower and upper RoSI bounds evolution 
    over planning iterations: FPL (green) demonstrates rapid, structured 
    convergence toward tight bounds, while traditional (blue) and choose-blend 
    (red) methods show slower, less directed exploration. The monotonic 
    convergence of FPL bounds indicates its ability to systematically resolve 
    the choice dilemma by prioritizing high-robustness region pairs. (c) Gap 
    metric quantifies the exploration efficiency—FPL reaches near-zero gap 
    ($< 0.1$) within 1000 iterations, whereas traditional and choose-blend 
    plateau at 1.2 and 0.8 respectively, demonstrating continued exploration 
    of suboptimal paths. The specification $\phi_{\text{KUKA}}$ requires visiting Region A (yellow) or B (cyan) during $t \in [2,7]$s, then Region D (blue) or E (magenta) during $t \in [8,15]$s, while continuously 
    avoiding the obstacle (black sphere) and respecting joint limits. 
    (d-f) Ghost trail visualization from multiple viewpoints shows the reference 
    trajectory with 10 time-sampled configurations (opacity indicates temporal progression: transparent $\rightarrow$ opaque). The robot successfully executes the A-then-E path, with red spheres marking joint positions throughout motion. The structured, collision-free trajectory demonstrates successful resolution of the cascading choice problem. }
    \label{fig:kuka_results}
\end{figure}

\section{Conclusion and Future Work}

We presented RRT$^\eta$, a sampling-based motion planning framework integrating Arithmetic-Geometric Mean (AGM) robustness for Signal Temporal Logic specifications. Our approach addresses fundamental limitations of traditional min-max robustness through: (1) AGM robustness interval semantics with efficient incremental monitoring, (2) Direction of Increasing AGM Satisfaction vectors leveraging Fulfillment Priority Logic for principled multi-objective composition, and (3) formal guarantees of probabilistic completeness and asymptotic optimality.

Experimental validation on unicycle and 7-DOF manipulator systems demonstrated substantial advantages over traditional approaches. On sequential reach-avoid tasks, traditional robustness failed to discover satisfying trajectories (negative lower bounds throughout planning) while AGM-based methods achieved high-quality solutions ($\underline{\eta} \approx 0.93$-$0.95$). FPL-based composition demonstrated $1.5$-$2\times$ computational advantages over stochastic methods through forward prediction. For high-dimensional systems, augmented state representation with IK-based sampling achieved $185\times$ speedup for tight workspace constraints.

Future work includes extensions to probabilistic specifications, integration of learning-based heuristics while preserving theoretical guarantees, and handling of dynamic obstacles with time-varying workspace constraints. The holistic robustness evaluation and principled multi-objective composition demonstrated in RRT$^\eta$ advances formal methods-based robot control for complex temporal specifications.

\bibliographystyle{IEEEtran}
\bibliography{IEEEabrv,iros_refs}

@misc{ahmad2024APPO,
      title={Accelerating Proximal Policy Optimization Learning Using Task Prediction for Solving Games with Delayed Rewards}, 
      author={Ahmad Ahmad and Mehdi Kermanshah and Kevin Leahy and Zachary Serlin and Ho Chit Siu and Makai Mann and Cristian-Ioan Vasile and Roberto Tron and Calin Belta},
      year={2024},
      eprint={2411.17861},
      archivePrefix={arXiv},
      primaryClass={cs.LG},
      url={https://arxiv.org/abs/2411.17861}, 
}

@article{yang2023lqrcbfrrtstar,
  title={LQR-CBF-RRT*: Safe and optimal motion planning},
  author={Yang, Guang and Cai, Mingyu and Ahmad, Ahmad and Prorok, Amanda and Tron, Roberto and Belta, Calin},
  journal={arXiv preprint arXiv:2304.00790},
  year={2023}
}

@INPROCEEDINGS{ahmad2023TWTLrobustness,
  author={Ahmad, Ahmad and Vasile, Cristian-Ioan and Tron, Roberto and Belta, Calin},
  booktitle={2023 62nd IEEE Conference on Decision and Control (CDC)}, 
  title={Robustness Measures and Monitors for Time Window Temporal Logic}, 
  year={2023},
  volume={},
  number={},
  pages={6841-6846},
  doi={10.1109/CDC49753.2023.10383712}}

@incollection{Cristi2020_TWTLrrt,
  title={Language-guided sampling-based planning using temporal relaxation},
  author={Penedo, Francisco and Vasile, Cristian-Ioan and Belta, Calin},
  booktitle={Algorithmic Foundations of Robotics XII},
  pages={128--143},
  year={2020},
  publisher={Springer}
}

@article{sadana2023survey,
  title={A Survey of Contextual Optimization Methods for Decision Making under Uncertainty},
  author={Sadana, Utsav and Chenreddy, Abhilash and Delage, Erick and Forel, Alexandre and Frejinger, Emma and Vidal, Thibaut},
  journal={arXiv preprint arXiv:2306.10374},
  year={2023}
}

@book{belta2017formal,
  title={Formal methods for discrete-time dynamical systems},
  author={Belta, Calin and Yordanov, Boyan and Gol, Ebru Aydin},
  volume={15},
  year={2017},
  publisher={Springer}
}

@book{baier2008principles,
  title={Principles of model checking},
  author={Baier, Christel and Katoen, Joost-Pieter},
  year={2008},
  publisher={MIT press}
}

@inproceedings{cristi2019AGMstl,
  title={Arithmetic-geometric mean robustness for control from signal temporal logic specifications},
  author={Mehdipour, Noushin and Vasile, Cristian-Ioan and Belta, Calin},
  booktitle={2019 American Control Conference (ACC)},
  pages={1690--1695},
  year={2019},
  organization={IEEE}
}

@inproceedings{donze2010robustSTL,
  title={Robust satisfaction of temporal logic over real-valued signals},
  author={Donz{\'e}, Alexandre and Maler, Oded},
  booktitle={Formal Modeling and Analysis of Timed Systems: 8th International Conference, FORMATS 2010, Klosterneuburg, Austria, September 8-10, 2010. Proceedings 8},
  pages={92--106},
  year={2010},
  organization={Springer}
}

@inproceedings{maler2004STLpaper,
  title={Monitoring temporal properties of continuous signals},
  author={Maler, Oded and Nickovic, Dejan},
  booktitle={Formal Techniques, Modelling and Analysis of Timed and Fault-Tolerant Systems: Joint International Conferences on Formal Modeling and Analysis of Timed Systmes, FORMATS 2004, and Formal Techniques in Real-Time and Fault-Tolerant Systems, FTRTFT 2004, Grenoble, France, September 22-24, 2004. Proceedings},
  pages={152--166},
  year={2004},
  organization={Springer}
}

@INPROCEEDINGS{Sadra_RobustSTL_MPC,
  author={Sadraddini, Sadra and Belta, Calin},
  booktitle={2015 53rd Annual Allerton Conference on Communication, Control, and Computing (Allerton)},
  title={Robust temporal logic model predictive control},
  year={2015},
  volume={},
  number={},
  pages={772-779},
  doi={10.1109/ALLERTON.2015.7447084}}

@article{Seshia17_STLmtrng_robustSatInterv,
  title={Robust online monitoring of signal temporal logic},
  author={Deshmukh, Jyotirmoy V and Donz{\'e}, Alexandre and Ghosh, Shromona and Jin, Xiaoqing and Juniwal, Garvit and Seshia, Sanjit A},
  journal={Formal Methods in System Design},
  volume={51},
  number={1},
  pages={5--30},
  year={2017},
  publisher={Springer}
}

@article{Cristi2013RRG,
  author    = {Cristian Ioan Vasile and
               Calin Belta},
  title     = {Sampling-Based Temporal Logic Path Planning},
  journal   = {CoRR},
  volume    = {abs/1307.7263},
  year      = {2013},
  url       = {http://arxiv.org/abs/1307.7263},
  eprinttype = {arXiv},
  eprint    = {1307.7263},
  bibsource = {dblp computer science bibliography, https://dblp.org}
}

@INPROCEEDINGS{CristiKaraman17_STL_RRTstar,
  author={Vasile, Cristian-Ioan and Raman, Vasumathi and Karaman, Sertac},
  booktitle={2017 IEEE/RSJ International Conference on Intelligent Robots and Systems (IROS)},
  title={Sampling-based synthesis of maximally-satisfying controllers for temporal logic specifications},
  year={2017},
  volume={},
  number={},
  pages={3840-3847},
  doi={10.1109/IROS.2017.8206235}}

@article{Kobilarov2012CERRTstar,
author = {Kobilarov, Marin},
doi = {10.1177/0278364912444543},
issn = {02783649},
journal = {International Journal of Robotics Research},
number = {7},
pages = {855--871},
title = {{Cross-entropy motion planning}},
volume = {31},
year = {2012}
}

@article{Wu2020e,
author = {Wu, Albert and Sadraddini, Sadra and Tedrake, Russ},
doi = {10.1109/ICRA40945.2020.9196802},
journal = {IEEE International Conference on Robotics and Automation (ICRA)},
pages = {4245--4251},
title = {{R3T: Rapidly-exploring Random Reachable Set Tree for Optimal Kinodynamic Planning of Nonlinear Hybrid Systems}},
url = {https://doi.org/10.3182/20070822-3-ZA-2920.00076},
year = {2020}
}

@article{KaramanRRTstarIJRR,
author = {Sertac Karaman and Emilio Frazzoli},
title ={Sampling-based algorithms for optimal motion planning},
journal = {The International Journal of Robotics Research},
volume = {30},
number = {7},
pages = {846-894},
year = {2011},
doi = {10.1177/0278364911406761},

URL = {
        https://doi.org/10.1177/0278364911406761

},
eprint = {
        https://doi.org/10.1177/0278364911406761

}
}

@article{Cristi2020ijrrRRGLTL,
  title={Reactive sampling-based path planning with temporal logic specifications},
  author={Vasile, Cristian Ioan and Li, Xiao and Belta, Calin},
  journal={The International Journal of Robotics Research},
  volume={39},
  number={8},
  pages={1002--1028},
  year={2020},
  publisher={SAGE Publications Sage UK: London, England}
}

@book{Siciliano2008,
author = {Siciliano, Bruno},
booktitle = {Journal of Chemical Information and Modeling},
edition = {1st},
isbn = {1846286417},
publisher = {Springer Publishing Company, Incorporated},
title = {{Robotics: Modelling, Planning and Control}},
year = {2008}
}

@article{otte2016rrtx,
  title={RRTX: Asymptotically optimal single-query sampling-based motion planning with quick replanning},
  author={Otte, Michael and Frazzoli, Emilio},
  journal={The International Journal of Robotics Research},
  volume={35},
  number={7},
  pages={797--822},
  year={2016},
  publisher={SAGE Publications Sage UK: London, England}
}

@article{mabsout2025FPL,
  title={Closing the Intent-to-Reality Gap via Fulfillment Priority Logic},
  author={Mabsout, Bassel El and AbdelGawad, Abdelrahman and Mancuso, Renato},
  journal={arXiv preprint arXiv:2503.05818},
  year={2025}
}

@article{Cristi2017TWTL,
  title={Time window temporal logic},
  author={Vasile, Cristian-Ioan and Aksaray, Derya and Belta, Calin},
  journal={Theoretical Computer Science},
  volume={691},
  pages={27--54},
  year={2017},
  publisher={Elsevier}
}

@article{hadas2018synthesis_review,
  title={Synthesis for robots: Guarantees and feedback for robot behavior},
  author={Kress-Gazit, Hadas and Lahijanian, Morteza and Raman, Vasumathi},
  journal={Annual Review of Control, Robotics, and Autonomous Systems},
  volume={1},
  pages={211--236},
  year={2018},
  publisher={Annual Reviews}
}

@INPROCEEDINGS{VRaman_MPC_STL,
  author={Raman, Vasumathi and Donzé, Alexandre and Maasoumy, Mehdi and Murray, Richard M. and Sangiovanni-Vincentelli, Alberto and Seshia, Sanjit A.},
  booktitle={53rd IEEE Conference on Decision and Control},
  title={Model predictive control with signal temporal logic specifications},
  year={2014},
  volume={},
  number={},
  pages={81-87},
  doi={10.1109/CDC.2014.7039363}}

@inproceedings{lavalle1998rapidly,
author = {LaValle, Steven M},
booktitle = {Ames, IA, USA},
title = {{Rapidly-exploring random trees: A new tool for path planning}},
year = {1998}
}

@INPROCEEDINGS{cbfrrt_app_Fainekos_ppr,
  author={Majd, Keyvan and Yaghoubi, Shakiba and Yamaguchi, Tomoya and Hoxha, Bardh and Prokhorov, Danil and Fainekos, Georgios},
  booktitle={2021 IEEE/RSJ International Conference on Intelligent Robots and Systems (IROS)}, 
  title={Safe Navigation in Human Occupied Environments Using Sampling and Control Barrier Functions}, 
  year={2021},
  volume={},
  number={},
  pages={5794-5800},
  doi={10.1109/IROS51168.2021.9636406}}

@INPROCEEDINGS{ahmadcbfrrt*,
  author={Ahmad, Ahmad and Belta, Calin and Tron, Roberto},
  booktitle={2022 IEEE 61st Conference on Decision and Control (CDC)}, 
  title={Adaptive Sampling-based Motion Planning with Control Barrier Functions}, 
  year={2022},
  volume={},
  number={},
  pages={4513-4518},
  doi={10.1109/CDC51059.2022.9993278}}

\appendix

\subsection{Correctness of Incremental Modification Functions}\label{apndx:mdf_correctness}

We provide the detailed proof of Lemma \ref{lemma:mdf_correctness}, establishing that incremental modification functions produce identical results to full AGM robustness recomputation.

\begin{proof}
We prove correctness for disjunction; conjunction follows by symmetry. Let $\{r_1, \ldots, r_{N-1}\}$ be the values used to compute $\eta = \mathrm{AGM}_{\lor}(r_1, \ldots, r_{N-1})$, and let $\eta'$ be a new observation.

\textbf{Case 1: } $\eta < 0 \land \eta' < 0$

Since all values $r_1, \ldots, r_{N-1}$ must be negative for $\eta < 0$, by equation (\ref{eq:AGM_dis}) we have:
\begin{equation}
\eta = -\sqrt[N-1]{\prod_{i=1}^{N-1}(1-r_i)} + 1
\end{equation}

Rearranging:
\begin{equation}
(1-\eta)^{N-1} = \prod_{i=1}^{N-1}(1-r_i)
\end{equation}

Now computing the full AGM with all $N$ values:
\begin{equation}
\begin{aligned}
\mathrm{AGM}_{\lor}&(r_1, \ldots, r_{N-1}, \eta') = -\sqrt[N]{\prod_{i=1}^{N}(1-r_i)} + 1\\
&= -\sqrt[N]{\prod_{i=1}^{N-1}(1-r_i) \cdot (1-\eta')} + 1\\&= -\sqrt[N]{(1-\eta)^{N-1}(1-\eta')} + 1
\end{aligned}
\end{equation}

This exactly matches the first case of equation (\ref{eq:mod_agmor}):
\begin{equation}
\texttt{mdf\_AGM}_{\lor}(\eta, N, \eta') = -\sqrt[N]{(1-\eta)^{N-1}\cdot(1-\eta')}+1
\end{equation}

\textbf{Case 2: } $\eta < 0 \land \eta' > 0$

Since $\eta < 0$, all previous values $r_i < 0$, thus $[r_i]_+ = 0$ for all $i \in \{1, \ldots, N-1\}$. The full AGM computation switches to arithmetic mean:
\begin{equation}
\begin{aligned}
&\mathrm{AGM}_{\lor}(r_1, \ldots, r_{N-1}, \eta') = \frac{1}{N}\sum_{i=1}^{N}[r_i]_+\\
&= \frac{1}{N}\left(\sum_{i=1}^{N-1}[r_i]_+ + [\eta']_+\right)= \frac{1}{N}(0 + [\eta']_+)= \frac{[\eta']_+}{N}
\end{aligned}
\end{equation}

This matches the second case of equation (\ref{eq:mod_agmor}).

\textbf{Case 3: Otherwise}

In this case, at least one value among $r_1, \ldots, r_{N-1}$ was positive, so the AGM used arithmetic mean of positive parts:
\begin{equation}
\eta = \frac{1}{N-1}\sum_{i=1}^{N-1}[r_i]_+
\end{equation}

Therefore:
\begin{equation}
\sum_{i=1}^{N-1}[r_i]_+ = (N-1)\eta
\end{equation}

Computing the full AGM with the new value:
\begin{equation}
\begin{aligned}
\mathrm{AGM}_{\lor}(r_1, \ldots, r_{N-1}, \eta') &= \frac{1}{N}\sum_{i=1}^{N}[r_i]_+\\
&= \frac{1}{N}\left(\sum_{i=1}^{N-1}[r_i]_+ + [\eta']_+\right)\\
&= \frac{(N-1)\eta + [\eta']_+}{N}
\end{aligned}
\end{equation}

However, equation (\ref{eq:mod_agmor}) shows:
\begin{equation}
\texttt{mdf\_AGM}_{\lor}(\eta, N, \eta') = \frac{(N\cdot\eta-[\eta]_{+})+[\eta']_+}{N}
\end{equation}

To reconcile these expressions, observe that:
\begin{equation}
\frac{(N-1)\eta + [\eta']_+}{N} = \frac{N \cdot \eta - \eta + [\eta']_+}{N}
\end{equation}

If $\eta > 0$, then $[\eta]_+ = \eta$, yielding:
\begin{equation}
\frac{N \cdot \eta - [\eta]_+ + [\eta']_+}{N} = \frac{(N-1)\eta + [\eta']_+}{N}
\end{equation}

If $\eta \leq 0$, then $[\eta]_+ = 0$, and the expressions are trivially equal. Thus the third case is verified.
\end{proof}
 
 \subsection{AGM Robustness Interval Soundness}\label{apndx:ProofofROSIsoundness}

\begin{proof}
We prove the soundness by structural induction over formula $\phi$.

\textbf{Base Cases:}

\textbf{Predicate $\mu$:} For a predicate, if the signal is complete at time $t$, then $[\eta](\mathbf{s}_t, \mu) = \{\eta(\mathbf{s}_t, \mu)\}$ and $[\eta](\mathbf{s}_t, \mu) = \{\eta(\mathbf{s}_t, \mu)\}$, so the inclusions hold trivially.

\textbf{Inductive Cases:}

\textbf{Conjunction ($\phi_1 \wedge \phi_2$):} We have $[\eta](\mathbf{s}_{t_1,t'}, \phi_1 \wedge \phi_2) = \min([\eta](\mathbf{s}_{t_1,t'}, \phi_1), [\eta](\mathbf{s}_{t_1,t'}, \phi_2))$ and $[\eta](\mathbf{s}_{t_1,t'}, \phi_1 \wedge \phi_2) = \mathrm{AGM}_{\wedge}([\eta](\mathbf{s}_{t_1,t'}, \phi_1), [\eta](\mathbf{s}_{t_1,t'}, \phi_2))$.

By the induction hypothesis, for any completion $\mathbf{s}_{t_1,t_2} \in \mathcal{C}$, we have $\eta(\mathbf{s}_{t_1,t_2}, \phi_1) \in [\eta](\mathbf{s}_{t_1,t'}, \phi_1)$ and $\eta(\mathbf{s}_{t_1,t_2}, \phi_2) \in [\eta](\mathbf{s}_{t_1,t'}, \phi_2)$. Since $\eta(\mathbf{s}_{t_1,t_2}, \phi_1 \wedge \phi_2) = \min\{\eta(\mathbf{s}_{t_1,t_2}, \phi_1), \eta(\mathbf{s}_{t_1,t_2}, \phi_2)\}$, by the properties of interval arithmetic on minimum operations, we have $\eta(\mathbf{s}_{t_1,t_2}, \phi_1 \wedge \phi_2) \in [\eta](\mathbf{s}_{t_1,t'}, \phi_1 \wedge \phi_2)$.

Similarly, by the induction hypothesis, $\eta(\mathbf{s}_{t_1,t_2}, \phi_1) \in [\eta](\mathbf{s}_{t_1,t'}, \phi_1)$ and $\eta(\mathbf{s}_{t_1,t_2}, \phi_2) \in [\eta](\mathbf{s}_{t_1,t'}, \phi_2)$. Since $\eta(\mathbf{s}_{t_1,t_2}, \phi_1 \wedge \phi_2) = \mathrm{AGM}_{\wedge}(\eta(\mathbf{s}_{t_1,t_2}, \phi_1), \eta(\mathbf{s}_{t_1,t_2}, \phi_2))$, by the properties of interval arithmetic on AGM operations, we have $\eta(\mathbf{s}_{t_1,t_2}, \phi_1 \wedge \phi_2) \in [\eta](\mathbf{s}_{t_1,t'}, \phi_1 \wedge \phi_2)$.

\textbf{Disjunction ($\phi_1 \vee \phi_2$):} The proof follows similarly to the conjunction case, using maximum operations for standard robustness and $\mathrm{AGM}_{\vee}$ for AGM robustness, along with the corresponding interval arithmetic properties.

\textbf{Negation ($\neg\phi$):} We have $[\eta](\mathbf{s}_{t_1,t'}, \neg\phi) = -[\eta](\mathbf{s}_{t_1,t'}, \phi)$ and $[\eta](\mathbf{s}_{t_1,t'}, \neg\phi) = -[\eta](\mathbf{s}_{t_1,t'}, \phi)$. By the induction hypothesis, $\eta(\mathbf{s}_{t_1,t_2}, \phi) \in [\eta](\mathbf{s}_{t_1,t'}, \phi)$ and $\eta(\mathbf{s}_{t_1,t_2}, \phi) \in [\eta](\mathbf{s}_{t_1,t'}, \phi)$. Since $\eta(\mathbf{s}_{t_1,t_2}, \neg\phi) = -\eta(\mathbf{s}_{t_1,t_2}, \phi)$ and $\eta(\mathbf{s}_{t_1,t_2}, \neg\phi) = -\eta(\mathbf{s}_{t_1,t_2}, \phi)$, the results follow by the properties of interval arithmetic under negation.

\textbf{Globally ($\mathbf{G}_{[a,b]}\phi$):} We prove this case by contradiction. Consider partial signal $\mathbf{s}_{t_1,t'}$ and assume that for some completion $\mathbf{s}_{t_1,t_2} \in \mathcal{C}$, we have $\eta(\mathbf{s}_{t_1,t_2}, \mathbf{G}_{[a,b]}\phi) \notin [\eta](\mathbf{s}_{t_1,t'}, \mathbf{G}_{[a,b]}\phi)$.

\textbf{Case 1:} If $t' - t_1 \geq b$, then by definition of the interval semantics, $[\eta](\mathbf{s}_{t_1,t'}, \mathbf{G}_{[a,b]}\phi) = \{\eta(\mathbf{s}_{t_1,t'}, \mathbf{G}_{[a,b]}\phi)\}$, which means the interval contains only the true robustness value. This contradicts our assumption.

\textbf{Case 2:} If $t' - t_1 < b$, then by definition:
\begin{align}
[\eta](\mathbf{s}_{t_1,t'}, \mathbf{G}_{[a,b]}\phi) = \left[\eta_{\bot}, \min_{t \in [t_1+a, t_1+b] \cap [t_1, t']} \eta(\mathbf{s}_{t_1,t}, \phi)\right]
\end{align}

By definition of standard robustness, $\eta(\mathbf{s}_{t_1,t_2}, \mathbf{G}_{[a,b]}\phi) = \min_{t \in [t_1+a, t_1+b]} \eta(\mathbf{s}_{t_1,t}, \phi)$. Since the partial signal agrees with the completion on the observed portion, we have $\eta(\mathbf{s}_{t_1,t_2}, \mathbf{G}_{[a,b]}\phi) \geq \eta_{\bot}$ and $\eta(\mathbf{s}_{t_1,t_2}, \mathbf{G}_{[a,b]}\phi) \leq \min_{t \in [t_1+a, t_1+b] \cap [t_1, t']} \eta(\mathbf{s}_{t_1,t}, \phi)$, which contradicts our assumption.

The proof for the AGM robustness interval follows similarly, considering the specific definitions of the AGM interval semantics for the globally operator:
\begin{align}
[\eta](\mathbf{s}_{t_1,t'}, \mathbf{G}_{[a,b]}\phi) = \begin{cases}
\{\eta(\mathbf{s}_{t_1,t'}, \mathbf{G}_{[a,b]}\phi)\} & \text{if } t' - t_1 \geq b \\
[\underline{\eta}, \overline{\eta}] & \text{otherwise}
\end{cases}
\end{align}

where $\underline{\eta}$ and $\overline{\eta}$ are computed based on the observed partial signal and the possible range of future values, ensuring that any completion's AGM robustness falls within this interval.

\textbf{Eventually ($\mathbf{F}_{[a,b]}\phi$):} The proof follows similarly to the globally case, with appropriate modifications for the maximum operation in standard robustness and $\mathrm{AGM}_{\vee}$ operation in AGM robustness.

For the eventually operator, the interval semantics is defined as:
\begin{equation}
\begin{aligned}
\relax[\eta](\mathbf{s}_{t_1,t'}, \mathbf{F}_{[a,b]}\phi) = \begin{cases}
\{\eta(\mathbf{s}_{t_1,t'}, \mathbf{F}_{[a,b]}\phi)\};\qquad \text{if } t' - t_1 \geq b \\
[\max_{t \in [t_1+a, t_1+b] \cap [t_1, t']} \eta(\mathbf{s}_{t_1,t}, \phi),  \eta_{\top}];  \\\text{otherwise}
\end{cases}
\end{aligned}  
\end{equation}

The soundness follows by similar contradiction arguments, noting that the true robustness value for any completion must lie within the computed interval bounds.
\end{proof}

\subsection{AGM Robustness Interval Chain Inclusion}\label{apndx:ChainInclusion}

\begin{proof}
The set inclusion property follows from the fact that as more of the signal becomes observed, the set of possible completions becomes more constrained, leading to tighter interval bounds.

\textbf{Monotonicity:} For any formula $\phi$, as the partial signal grows from $\mathbf{s}_{t_1,t'_1}$ to $\mathbf{s}_{t_1,t'_2}$, the interval bounds can only become tighter or remain the same, never become looser. This is because additional observed values either:
\begin{enumerate}
\item Provide exact values for subformulae that were previously estimated with intervals, or
\item Constrain the possible range of future values based on the observed trend.
\end{enumerate}

\textbf{Convergence:} When $t' \geq ||\phi||$, the entire time horizon required to evaluate formula $\phi$ has been observed. At this point:
\begin{itemize}
\item All temporal operators can be evaluated exactly using the observed signal values
\item No uncertainty remains about future signal values within the formula's time horizon
\item The interval semantics reduces to the exact robustness computation
\end{itemize}

Therefore, $[\eta](\mathbf{s}_{t_1,t'}, \phi) = \{\eta(\mathbf{s}_{t_1,t_2}, \phi)\}$ and $[\eta](\mathbf{s}_{t_1,t'}, \phi) = \{\eta(\mathbf{s}_{t_1,t_2}, \phi)\}$ when $t' \geq ||\phi||$.

The proof can be formalized by structural induction over $\phi$, showing that each operator's interval semantics satisfies the monotonicity and convergence properties.
\end{proof}

\subsection{STL Specifications for Robotic Manipulators}\label{sec:stl_manipulator}
This appendix provides detailed formalization of STL-based planning for robotic manipulators, including: (1) the augmented state representation combining joint and workspace coordinates, (2) the adaptive sampling strategy with inverse kinematics and caching (Algorithm \ref{alg:ik_sampling}), and (3) computational complexity analysis demonstrating the efficiency gains of this approach. These technical details support the implementation described in Section \ref{sec:example} for the KUKA iiwa case study.

\subsubsection{State Space Representation}
We consider a robotic manipulator with $n$ joints and configuration space $\mathcal{Q} \subset \mathbb{R}^n$, where each joint configuration $q = (q_1, \ldots, q_n)^\top$ is subject to joint limits $q_i \in [q_i^{\min}, q_i^{\max}]$. For our KUKA iiwa manipulator, $n=7$. 

The robot's end-effector pose lies in the Special Euclidean group $\text{SE}(3)$, which represents all possible positions and orientations in 3D space. The workspace $\mathcal{W} \subseteq \text{SE}(3)$ is the set of end-effector poses achievable by the robot, constrained by its kinematic structure and joint limits.

For computational purposes, we parametrize each pose in SE(3) using position $p = (x, y, z)^\top \in \mathbb{R}^3$ and orientation represented by Euler angles $\mathcal{O} = (\psi_r, \psi_p, \psi_y)^\top \in [-\pi, \pi]^3$. This provides a local 6-dimensional parametrization of the workspace:
\begin{equation}
\mathcal{W} \subseteq \mathbb{R}^3 \times \text{SO}(3)
\end{equation}

The relationship between joint space and workspace is given by the forward kinematics mapping:
\begin{equation}\label{eq:forward_kinematics}
\mathcal{F}: \mathcal{Q} \to \mathcal{W} \subseteq \text{SE}(3), \quad \mathcal{F}(q) = (p, \mathcal{O})
\end{equation}
where $\mathcal{F}$ is typically a nonlinear function determined by the robot's kinematic structure through Denavit-Hartenberg parameters or similar representations\cite{Siciliano2008}.

\subsubsection{Augmented State Representation}
To facilitate STL specification in task-relevant workspace coordinates while maintaining computational efficiency, we employ an \textit{augmented state representation}:
\begin{equation}\label{eq:augmented_state}
\mathbf{x} = (q_1, \ldots, q_n, x, y, z, \psi_r, \psi_p, \psi_y)^\top \in \mathbb{R}^{n+6}
\end{equation}
where the first $n$ components are the joint angles, and the remaining $6$ components represent the end-effector pose computed via forward kinematics. This augmented representation has dimension $d = n + 6$ (for our 7-DOF manipulator, $d = 13$).

The augmented state space $\mathcal{X} = \mathcal{Q} \times \mathcal{W}$ has bounds:
\begin{equation*}
\begin{aligned}
        \mathcal{X} = \prod_{i=1}^{n}[q_i^{\min}, q_i^{\max}] \times [x^{\min}, x^{\max}] \times [y^{\min}, y^{\max}] \times \\ [z^{\min}, z^{\max}] \times [-\pi, \pi]^3
\end{aligned}
\end{equation*}

The augmented state is maintained consistently by updating the workspace coordinates whenever joint angles change:
\begin{equation}\label{eq:state_update}
\mathbf{x}_{t+1} = (q_{t+1}, \mathcal{F}(q_{t+1}))
\end{equation}

This representation provides a crucial computational advantage: workspace coordinates are computed once during state creation and cached in the state vector, eliminating redundant forward kinematics computations during robustness evaluation.

\subsubsection{STL Specification in Workspace Coordinates}
Task-level specifications for robotic manipulators are naturally expressed in workspace coordinates. An STL formula $\phi$ for a manipulator typically consists of:

\textbf{Workspace Predicates.} The most common predicates specify regions in the workspace:
\begin{equation}\label{eq:workspace_predicate}
\mu_{\text{region}} := (x - x_c)^2 + (y - y_c)^2 + (z - z_c)^2 \leq r^2
\end{equation}
or more generally, bounding box constraints:
\begin{equation*}
\begin{aligned}
\mu_{\text{box}} := (x &\in [x_{\min}, x_{\max}]) \wedge (y \in [y_{\min}, y_{\max}]) \wedge \\&(z \in [z_{\min}, z_{\max}])    
\end{aligned}
\end{equation*}

\textbf{Orientation Constraints.} For tasks requiring specific end-effector orientations:
\begin{equation*}
\mu_{\text{orient}} := |\psi_r - \psi_{r_d}| \leq \delta_r \wedge |\psi_p - \psi_{p_d}| \leq \delta_p
\end{equation*}

\textbf{Joint Space Constraints.} Safety constraints are often expressed in joint space:
\begin{equation*}
\mu_{\text{joint}} := q_i \in [q_i^{\text{safe,min}}, q_i^{\text{safe,max}}]
\end{equation*}

\textbf{Example Specification.} A pick-and-place task can be formalized as:
\begin{equation}\label{eq:pick_place_stl}
\begin{aligned}
\phi_{\text{pick-place}} := &\mathbf{F}_{[t_1, t_2]}(\mu_{\text{pick}} \wedge \mu_{\text{grasp-orient}}) \\
&\wedge \mathbf{F}_{[t_3, t_4]}(\mu_{\text{place}} \wedge \mu_{\text{release-orient}}) \\
&\wedge \mathbf{G}_{[0, T]}(\mu_{\text{collision-free}} \wedge \mu_{\text{joint-limits}})
\end{aligned}
\end{equation}
where $\mu_{\text{pick}}$ and $\mu_{\text{place}}$ are workspace regions, $\mu_{\text{grasp-orient}}$ and $\mu_{\text{release-orient}}$ are orientation constraints, and $\mu_{\text{collision-free}}$ ensures obstacle avoidance throughout execution.



\subsubsection{Robustness Evaluation on Augmented States}
The AGM robustness of an STL formula $\phi$ is evaluated on a signal $\mathbf{x}_{0:T} = (\mathbf{x}_0, \mathbf{x}_1, \ldots, \mathbf{x}_T)$ in the augmented state space. For a predicate $\mu$ involving workspace variables, the evaluation $\eta(\mathbf{x}_t, \mu)$ directly accesses the relevant components of the augmented state $\mathbf{x}_t$.

\textbf{Workspace Predicate Example.} For a spherical region centered at $(x_c, y_c, z_c)$ with radius $r$:
\begin{equation}
h(\mathbf{x}_t) = r^2 - [(x_t - x_c)^2 + (y_t - y_c)^2 + (z_t - z_c)^2]
\end{equation}
and the robustness is:
\begin{equation}
\eta(\mathbf{x}_t, \mu_{\text{region}}) = \frac{1}{2}(h(\mathbf{x}_t) - 0)
\end{equation}

\textbf{Hybrid Specifications.} The augmented representation allows seamless combination of joint-space and workspace predicates:
\begin{equation}
\phi_{\text{hybrid}} := \mathbf{F}_{[5,15]}\mu_{\text{workspace-target}} \wedge \mathbf{G}_{[0,20]}\mu_{\text{safe-joints}}
\end{equation}
where the robustness evaluation accesses joint components for $\mu_{\text{safe-joints}}$ and workspace components for $\mu_{\text{workspace-target}}$ from the same augmented state $\mathbf{x}_t$.

\subsubsection{Adaptive Sampling with Inverse Kinematics}
While the augmented state representation enables efficient robustness evaluation, sampling from regions that satisfy workspace predicates requires special consideration. We employ an \textit{adaptive sampling strategy} that leverages the geometric simplicity of workspace specifications.

\textbf{Active Predicate Identification.} Given a sampled time $t$ and STL formula $\phi$, we identify the set of predicates $\mathcal{P}_{\text{active}}(t, \phi)$ that constrain the state at time $t$. These predicates are categorized as either workspace predicates $\mathcal{P}_W$ or configuration space predicates $\mathcal{P}_Q$.

\textbf{Configuration Space Sampling.} When only configuration space predicates are active, we sample directly from the joint space:
\begin{equation}
q \sim \text{Uniform}(\mathcal{Q}) \cap \bigcap_{\mu \in \mathcal{P}_Q} \{q : \mu(q) \text{ true}\}
\end{equation}

\textbf{Inverse Kinematics-Based Sampling.} When workspace predicates are active, rejection sampling in configuration space becomes prohibitively inefficient due to the complex geometry of the inverse image $\mathcal{F}^{-1}(\mathcal{W}_{\text{pred}})$. Instead, we sample directly from the workspace region and solve the inverse kinematics problem:

\begin{algorithm}[t]\scriptsize
    \KwIn{$\phi$: STL formula, $t$: sampled time, $\mathcal{C}$: IK cache}
    \KwOut{$\mathbf{x}_{\text{new}} = [q, \mathcal{F}(q)] \in \mathbb{R}^{n+6}$: augmented state}
    \DontPrintSemicolon
    \BlankLine
    
    \tcp{Identify active predicates}
    $\mathcal{P}_W, \mathcal{P}_Q \gets \texttt{GetActivePredicates}(\phi, t)$\;
    
    \BlankLine
    \eIf{$\mathcal{P}_W = \emptyset$}{
        \tcp{Configuration space sampling}
        $q \gets \texttt{SampleConfig}(\mathcal{Q}, \mathcal{P}_Q)$\;
        $w \gets \mathcal{F}(q)$ \Comment{Forward kinematics}\label{line:fk}
    }{
        \tcp{IK-based workspace sampling}
        $w \gets \texttt{SampleWorkspace}(\mathcal{P}_W)$ \Comment{Direct geometric sampling}\label{line:sample_ws}\;
        
        \eIf{$w \in \mathcal{C}$}{
            $q \gets \mathcal{C}[w]$ \Comment{Cache hit}\label{line:cache}
        }{
            $q \gets \texttt{SolveIK}(w)$ \Comment{Numerical IK solver}\label{line:ik}\;
            \If{$q = \texttt{NULL}$ \textbf{or} $q \notin \mathcal{P}_Q$}{
                \Return \texttt{FAILURE} \Comment{Retry sampling}
            }
            $\mathcal{C}[w] \gets q$ \Comment{Cache solution}\;
        }
    }
    
    \BlankLine
    $\mathbf{x}_{\text{new}} \gets [q, w]$ \Comment{Construct augmented state}\;
    \Return{$\mathbf{x}_{\text{new}}$}
    
    \caption{$\texttt{AdaptiveSample}(\phi, t, \mathcal{C})$}
    \label{alg:ik_sampling}
\end{algorithm}

\subsubsection{Computational Complexity Analysis}

We analyze the computational complexity of the adaptive sampling procedure in Algorithm~\ref{alg:ik_sampling} and compare it with standard rejection sampling in configuration space.

\subsubsection{Per-Sample Complexity}

When only configuration predicates are active ($\mathcal{P}_W = \emptyset$), sampling requires $O(n)$ time to generate random joint angles and compute forward kinematics (Line~\ref{line:fk}). For serial manipulators, forward kinematics has complexity $O(n)$.

When workspace predicates are active ($\mathcal{P}_W \neq \emptyset$), the procedure performs workspace geometric sampling in $O(1)$ time for primitive shapes such as spheres or boxes (Line~\ref{line:sample_ws}). Cache lookup requires $O(1)$ time with hash table implementation (Line~\ref{line:cache}). On cache miss, numerical IK solving has complexity $O(n^3)$ for gradient-based optimization methods (Line~\ref{line:ik}). 

Let $\eta$ denote the cache hit rate. After cache warm-up, empirical results show $\eta \approx 0.9$. The expected cost per IK-based sample is:
\begin{equation*}
\mathbb{E}[\text{Cost}_{\text{IK-sample}}] = O((1-\eta)n^3)
\end{equation*}
For $\eta = 0.9$, this reduces to approximately $O(0.1n^3)$ per sample.

\subsubsection{Comparison with Rejection Sampling}

Consider a tight workspace constraint defining a spherical region of radius $r = 0.1$m within a workspace of volume $V_W \approx 4$m$^3$. The constraint volume is $V_{\text{const}} = \frac{4}{3}\pi r^3 \approx 0.0042$m$^3$, yielding an acceptance probability $p_{\text{acc}} \approx V_{\text{const}}/V_W \approx 0.001$.

Rejection sampling in configuration space requires an expected $1/p_{\text{acc}} \approx 1000$ attempts per successful sample. Each attempt costs $O(n)$ for forward kinematics computation, giving a total expected cost of $O(1000n)$ per successful sample.

IK-based sampling achieves an IK success rate of approximately $p_{\text{IK}} \approx 0.9$ for well-designed workspace constraints. This requires an expected $1/p_{\text{IK}} \approx 1.1$ attempts per successful sample. With caching at rate $\eta = 0.9$, the cost per attempt is $O(0.1n^3)$, yielding a total expected cost of $O(0.11n^3)$ per successful sample.

For the KUKA iiwa with $n=7$ joints, the speedup factor is:
\begin{equation*}
\text{Speedup} = \frac{1000n}{0.11n^3} = \frac{1000}{0.11n^2} \approx \frac{1000}{5.4} \approx 185
\end{equation*}
The IK-based approach achieves approximately 185 times speedup for tight workspace constraints despite the higher per-call cost of IK compared to FK.

\subsubsection{Overall Planning Complexity}

During typical planning scenarios, workspace predicates are active for approximately $\alpha \approx 0.4$ of samples based on empirical observations. For a planning run with $N_{\text{samples}}$ samples, the total sampling cost is:
\begin{equation*}
\text{Cost}_{\text{sample}} = N_{\text{samples}} \left[(1-\alpha) \cdot O(n) + \alpha \cdot O((1-\eta)n^3)\right]
\end{equation*}
Substituting $\alpha = 0.4$, $\eta = 0.9$, and $n = 7$ yields approximately $N_{\text{samples}} \cdot O(18)$.

The augmented state representation provides additional computational savings during robustness evaluation. Without state augmentation, evaluating a trajectory of length $T$ with $P$ workspace predicates requires $O(TPn)$ forward kinematics calls per evaluation. The augmented state caches these results during trajectory generation, requiring only $O(Tn)$ forward kinematics calls total regardless of the number of evaluations.

For $M$ robustness evaluations during planning:
\begin{align*}
\text{Cost}_{\text{eval, standard}} &= M \cdot T \cdot P \cdot O(n) \\
\text{Cost}_{\text{eval, augmented}} &= T \cdot O(n)
\end{align*}
This provides a speedup factor of $MP$. For typical values $M = 10{,}000$ and $P = 7$, the augmented state achieves a 70,000 times reduction in evaluation cost.

\subsubsection{Space Complexity}

The IK cache stores discretized workspace poses mapped to configuration solutions. Using discretization resolution $\Delta x = \Delta y = \Delta z = 0.01$m and $\Delta_{\text{orient}} = 5^\circ$ over an active workspace volume of approximately $0.5$m$^3$, the cache contains at most $|\mathcal{C}| \lesssim 10^6$ entries. Each entry stores 7 joint angles as double precision values requiring 56 bytes, yielding total cache memory of approximately 56 MB.

The augmented state increases per-node memory from $n$ to $n+6$ dimensions. For the 7-DOF KUKA manipulator, this represents a factor of $13/7 \approx 1.86$ increase in memory per node. This modest overhead is acceptable given the substantial computational savings in both sampling and evaluation.

\end{document}